%% file: main.tex
\documentclass[]{fairmeta}
\usepackage{microtype}
\usepackage{amsfonts}
\usepackage{graphicx}
\usepackage{booktabs}
\usepackage{wrapfig}
\usepackage{amsmath}
\usepackage{amsthm}
\usepackage{pifont}
\usepackage[capitalize,noabbrev]{cleveref}
\theoremstyle{plain}
\newtheorem{theorem}{Theorem}[section]
\newtheorem{proposition}[theorem]{Proposition}

\newtheorem{corollary}[theorem]{Corollary}
\theoremstyle{definition}
\newtheorem{definition}[theorem]{Definition}

\theoremstyle{remark}

\usepackage{enumitem}
\usepackage[subtle, mathdisplays=tight, charwidths=tight, leading=normal]{savetrees}
\usepackage{subcaption} 
\usepackage{booktabs} 
\usepackage{longtable} 
\usepackage{textcomp} 
\usepackage{xspace} 
\usepackage{hyperref} 

\newcommand{\Sys}{\textsc{Jackpot}\xspace}
\newcommand{\sys}{\textsc{Jackpot}\xspace}
\newcommand{\topk}{{Top}-$k$\xspace}
\newcommand{\qwentwob}{\textsc{Qwen3-1.7B-Base}\xspace}

\newcommand{\qweneightb}{\textsc{Qwen3-8B-Base}\xspace}

\usepackage[table]{xcolor} 

\usepackage{tikz} 

\usepackage{tikz} 
\usepackage{algorithm}
\usepackage{algorithmic}
\usepackage{adjustbox}

\title{Jackpot: Optimal Budgeted Rejection Sampling for Extreme Actor-Policy Mismatch Reinforcement Learning}

\definecolor{skyblue}{RGB}{135, 206, 235} 
\definecolor{palegreen}{RGB}{152, 251, 152}
\definecolor{lightpink}{RGB}{255, 220, 235}
\author{Zhuoming Chen$^*$, Hongyi Liu$^*$, Yang Zhou$^*$, Haizhong Zheng, Beidi Chen} 
\affiliation{Carnegie Mellon University}
\abstract{
Reinforcement learning (RL) for large language models (LLMs) remains expensive,
particularly because the rollout is expensive. Decoupling rollout generation from policy optimization  (e.g., leveraging a more efficient model to rollout) could enable 
substantial efficiency gains, yet doing so introduces severe distribution 
mismatch that destabilizes learning.  
We propose \sys, a framework that leverages Optimal Budget Rejection 
Sampling (OBRS) to directly reduce the discrepancy between the rollout model and 
the evolving policy. \sys integrates a principled OBRS procedure, a unified 
training objective that jointly updates the policy and rollout models, and an 
efficient system implementation enabled by top-$k$ probability estimation and 
batch-level bias correction.  
Our theoretical analysis shows that OBRS consistently moves the rollout 
distribution closer to the target distribution under a controllable acceptance 
budget. 
Empirically, \sys substantially improves training stability compared to 
importance-sampling baselines, achieving performance comparable to on-policy RL 
when training Qwen3-8B-Base for up to 300 update steps of batchsize 64. Taken together, our results show that OBRS-based alignment brings us a step 
closer to practical and effective decoupling of rollout generation from policy 
optimization for RL for LLMs. 
}

\metadata[Github]{\url{https://github.com/Infini-AI-Lab/jackpot}} 
\metadata[Website]{\url{https://infini-ai-lab.github.io/jpt_website/}} 

\begin{document}

\maketitle
\input{jackpot/introduction}

\input{jackpot/related} 
\input{jackpot/preliminary}

\input{jackpot/method}
\input{jackpot/empiricalstudies}
\input{jackpot/conclusion}

\clearpage
\newpage
\bibliographystyle{assets/plainnat}
\bibliography{paper}

\clearpage
\newpage
\beginappendix 
\input{jackpot/appendix}

\end{document}

%% file: jackpot/introduction.tex
\section{Introduction} 
\footnote{$^*$ indicates equal contribution, order decided by the lastnames} 
\vspace{-1em} 

Reinforcement learning (RL) has demonstrated substantial effectiveness in the post-training of large language models (LLMs), yielding significant improvements in domains such as mathematics~\citep{guo2025deepseek,azerbayev2023llemma}, coding~\citep{jimenez2023swe,ouyang2025kernelbench}, and other agentic tasks~\citep{liu2023agentbench,jin2025search}. Despite these successes, RL remains expensive~\citep{sheng2025hybridflow,fu2025areal,zheng2025act}, with the majority of the training cost, often  $80\%$~\citep{qin2025seer}, attributed to rollouts, during which LLMs auto-regressively generate trajectories. 

Many approaches have been explored to reduce rollout costs by improving hardware utilization through asynchronous training~\citep{zheng2025prosperity,fu2025areal,wu2025llamarl}, or through inference optimizations such as 8-bit quantization~\citep{liu2025flashrl}, request balancing~\citep{qin2025seer}, and speculative decoding~\citep{Atlas}. However, these methods all require the target LLM to actively participate in the rollout process to collect trajectories and signals, which ultimately limits their flexibility~\citep{kiran2021deep} and efficiency. This leads to a question:

\begin{tcolorbox}[myquote]
\itshape Is it possible to perform rollouts using a completely different model from the one we ultimately want to train?
\end{tcolorbox}

Ideally, switching to a smaller variant from the same model family should give huge efficiency gain. 
By prior work on test-time compute~\citep{brown2024large,sadhukhan2025kinetics}, even relatively small models can obtain non-zero rewards (i.e., training signals) on challenging problems by leveraging multiple attempts, a behavior that is inherently aligned with the standard RL training paradigm~\citep{guo2025deepseek}. 

\begin{figure}[h!]
    \centering
    \includegraphics[width=1.0\textwidth]{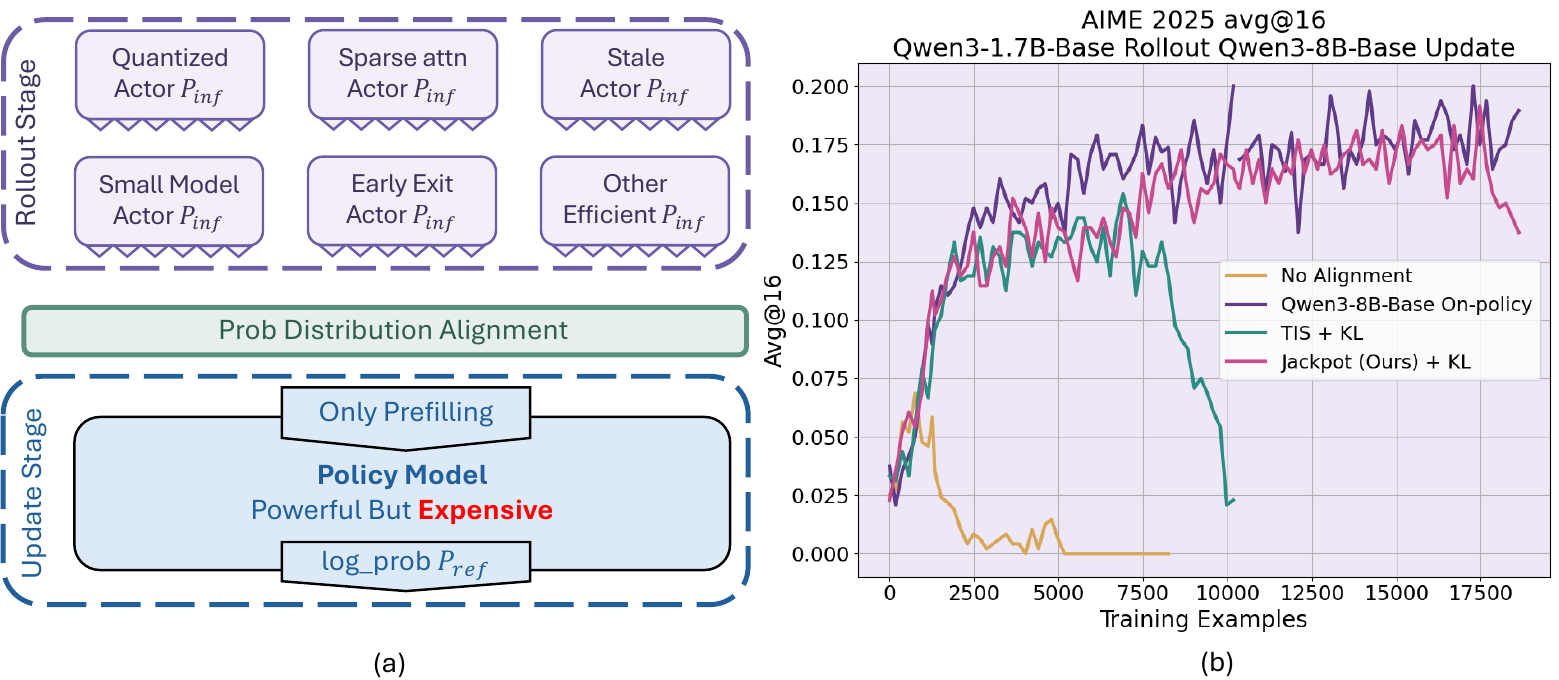} 
    \caption{(a) Off-policy RL training is desirable for alleviating the severely high cost from the rollout stage of on-policy RL. The strong but inference costly policy model can be approximated by cheap but crappy counterparts to speedup the rollout stage. However, the trajectories rollout by the actor model must be aligned in probability distribution with the policy model's distribution to make offpolicy training viable. In (b) we show an extreme off-policy case where we show training setting use a \qwentwob model training rollout to train a \qweneightb model policy. Without any alignment procedures, training collapses (pink). Prior method TIS (green) also shows a significant gap towards \qweneightb on-policy baseline (purple), while collapsing, using TIS sees KL divergence also violently increasing. Our proposed point \Sys provides much more stable training under the setting. } 
    \label{fig:rolloutvstrainingdiff}
\end{figure} 

Despite the availability of training signals, the decoupled RL training triggers a severe \textbf{distribution mismatch} between the rollout and trained models, which is known to severely affect the stability and convergence of RL~\citep{liu-li-2025-rl-collapse}, primarily due to inaccurate advantage estimates. Existing methods aim to mitigate this issue with various importance-sampling (IS) corrections~\citep{fu2025areal,wu2025llamarl,liu2025flashrl,zheng2025group,team2025every}. However, we find that the KL divergence between \emph{different} models is often an order of magnitude larger than that observed in asynchronous training or quantised variants, casting doubt on whether prior correction techniques remain effective under such a large mismatch.

Given the limitations of purely post-hoc corrections, it is desirable also to reduce the mismatch \emph{at the source}. Rejection sampling, which simulates a target distribution from an accessible proposal distribution, offers such a direct mechanism by selectively excluding undesired tokens from contributing to the backward pass and policy updates. It effectively narrows the distribution gap between the rollout and training models, naturally \emph{complementing} existing importance-sampling corrections.

However, applying rejection sampling in RL systems for LLMs poses three key technical challenges.  \ding{172} \textbf{Low Sample Efficiency.} 
Let $\boldsymbol{p}$ denote the trained model distribution and $\boldsymbol{q}$ denote the rollout model distribution. 
In standard rejection sampling, a token $i$ proposed from $\boldsymbol{q}$ is accepted with probability 
$\tfrac{p_i}{\lambda q_i}$, where the constant $\lambda$ must satisfy $\lambda \ge \max_{i} \tfrac{p_i}{q_i}$. 
For LLMs, which operate over vocabularies exceeding $100{,}000$ tokens, this requirement leads to an extremely large $\lambda$ because even small local differences between the two distributions can cause the ratio $\tfrac{p_i}{q_i}$ to spike on certain rare tokens. 
Consequently, the acceptance probability becomes vanishingly small for almost all proposed tokens.  \ding{173} \textbf{Widening Gap.} 
A naïve implementation of decoupled training keeps the rollout model fixed while the trained model continually improves. As learning progresses, the discrepancy between the two models increases, further exacerbating the distribution mismatch and potentially destabilising training. \ding{174} \textbf{Efficiency and System Support.} 
Current RL frameworks~\citep{sheng2025hybridflow,fu2025areal,vonwerra2022trl,hu2024openrlhf} assume that the rollout and training models are identical. It remains unclear how to efficiently support decoupled RL, especially when combined with new algorithms, objectives, or loss functions.

To tackle these challenges, we propose \sys. 
Our framework leverages optimal budget rejection sampling (OBRS)~\citep{verine2024optimal} as a relaxed alternative to classical rejection sampling. 
Although OBRS does not enforce exact equality between the actor and target policy distributions, it provides a provable guarantee that, for any specified rejection budget, the adjusted actor distribution becomes strictly closer to the target distribution than the original proposal. 
To prevent the distribution gap from widening as training progresses, we further apply a reverse-KL loss to progressively align the rollout model with the trained model. 
In addition, we develop an end-to-end RL system that efficiently supports decoupled training, incorporating several approximations to reduce the memory overhead introduced by OBRS, which otherwise requires accessing the full vocabulary when reweighting accepted tokens. Taken together, our empirical results show that \sys substantially improves the
stability of RL training compared to IS-based baselines, even enabling performance
comparable to on-policy training on \qweneightb for 300 training steps (\Cref{fig:rolloutvstrainingdiff}).

 We organize the remainder of this paper as follows.
\begin{itemize}[itemsep=0.0pt,topsep=0pt,leftmargin=*]
    \item In~\Cref{sec:related}, we formalize the notion of distribution mismatch between rollout and trained models, analyze its underlying causes across different training paradigms, and review relevant prior work.
    
    \item In~\Cref{sec:method}, we start from a brief introduction of optimal budget rejection sampling (\Cref{sec:concept}) and  demonstrate its effectiveness through both numerical simulations and empirical observations (\Cref{sec:numerical}) and present the theoretical optimality (\Cref{sec:theorem}).
    \item In~\Cref{sec:jackpot}, we present the full \sys framework, detailing three key components:  
    (i) the OBRS procedure (\Cref{subsec:reject_and_reweighting});  
    (ii) the formulation of the \sys training objective (\Cref{subsec:formulation}); and  
    (iii) an efficient system implementation enabled by \topk-based probability estimation and batch-level bias correction (\Cref{subsec:naive-concerns}).
    
    \item In~\Cref{sec:newexp}, we conduct experiments on mathematical reasoning tasks to evaluate \sys. Empirically, we show that \sys significantly improves the stability of RL training over IS-based baselines. We further perform ablations on each module in \sys and analyze its effectiveness in settings with smaller distribution gaps, such as large-batch training and KV-quantized rollout.
\end{itemize}

%% file: jackpot/related.tex
\section{Background}
\label{sec:related}
In this section, we first formalize the distribution mismatch problem that arises in RL for LLMs. We then review several strands of related work of \Sys. 

\subsection{Problem Setting: PPO Objective and Actor-Policy Distribution Mismatch} 

We begin with the clipped objective in PPO~\citep{schulman2017proximalpolicyoptimizationalgorithms}, whose expectation can be written as
\begin{equation}
    \mathcal{L}^{\text{{PPO}}}(\theta) 
    = \mathbb{E}_{x \sim P_{\text{inf}}}
    \Big[ \min \big( r_\theta(x)\,\hat{A}(x), 
    \operatorname{clip}(r_\theta(x), 1-\epsilon, 1+\epsilon)\,\hat{A}(x) \big) \Big] 
\end{equation} 
where $r_\theta(x) = {p_{\boldsymbol{\theta}_{\text{new}}}(x)}/{p_{\text{ref}}(x)}$ is the likelihood ratio between the updated policy $p_{\boldsymbol{\theta}_{\text{new}}}$ and the reference policy $p_{\text{ref}}$, and $\hat{A}(x)$ denotes the estimated advantage at decision $x$. $p_{\text{inf}}$ is the inference distribution used to generate rollouts, $p_{\text{ref}}$ is the reference policy distribution assumed in the objective, and $p_{\boldsymbol{\theta}_{\text{new}}}$ is the updated policy distribution. In the standard process, it is assumed that $p_{\text{inf}} = p_{\text{ref}}$, but in practice this assumption is often violated, leading to actor–policy distribution mismatch.

Distribution mismatch is common and arises for several reasons, such as minor discrepancies between the inference engine and the reference policy by FSDP engines, the use of stale or asynchronous data, or rollouts generated by approximated models (e.g., quantized, sparsified, or distilled). Such mismatches can destabilize training and therefore require additional mechanisms to correct or mitigate their impact.

\subsection{Related Work}
\textbf{RL for LLM.} Reinforcement learning has been widely applied to LLMs to improve human alignment, reasoning, coding, and other complex tasks. Beyond PPO, memory efficient methods have been proposed, including ReMax~\citep{li2023remax}, RLOO~\citep{ahmadian2024back}, and GRPO~\citep{shao2024deepseekmath}. In addition, methods such as SimPO~\citep{meng2024simpo} and DPO~\citep{rafailov2023direct}, which are based on offline RL, have also been employed for human alignment. RL training systems for LLMs, such as Verl~\citep{sheng2025hybridflow}, AReal~\citep{fu2025areal}, TRL~\citep{vonwerra2022trl}, and OpenRLHF~\citep{hu2024openrlhf}, have been developed to improve training throughput and scalability.

\textbf{Distribution Mismatch Correction in RL.} Actor (i.e., the rollout model) and policy (i.e., the trained model) mismatch is a common problem that has long been studied, e.g., \citet{espeholt2018impalascalabledistributeddeeprl}. To alleviate the actor-policy distribution gap, prior methods leverage importance sampling to approximate the true PPO objective.
\begin{equation}
    \mathcal{L}^{\text{PPO}}(\theta) 
    = \mathbb{E}_{x \sim p_{\text{inf}}} \Big[ \text{SG}\!\left(\colorbox{lightpink}{$\displaystyle\boldsymbol{F}(\tfrac{p_{\text{ref}}(x)}{p_{\text{inf}}(x)})$}\right)\min \big( r_\theta(x)\,\hat{A}(x), 
    \operatorname{clip}(r_\theta(x), 1-\epsilon, 1+\epsilon)\,\hat{A}(x) \big) \Big]
\end{equation} 

$\text{SG}$ means stop gradients. The choice of adjustment function $\boldsymbol{F}$ is one of the core focuses of prior work.  Areal~\citep{fu2025areal} uses an identical function $\boldsymbol{F}(x) = x$; Flash-RL and Llama-RL use $\boldsymbol{F}(x) = \min(x, C)$ (i.e., truncated importance sampling, TIS), where $C$ is a hyper-parameter;
IceProp~\citep{team2025every} uses a bi-directional truncation,
\[
\boldsymbol{F}(x) =
\begin{cases}
x, & \text{if } x \in [\alpha, \beta],\\[6pt]
0, & \text{otherwise}.
\end{cases}
\]

From system perspective, FP32 LM heads~\citep{liu2025flashrl} and deterministic LLM Inference~\citep{he2025nondeterminism} are implemented to mitigate the numerical issue of serving systems when rollout.

In this paper, we proposed \Sys.
Our method is \textbf{orthogonal} to the above prior works. We think about the problem of how to directly close the gap between  
$p_{\text{inf}}$, and $p_{\text{ref}}$. We will show that through optimal budget rejection sampling and reweighting of the output probabilities so that the
divergence between $p_{\text{inf}}$ and the target distributions $p_{\text{ref}}$ is provably
reduced. Moreover, techniques such as TIS can be applied on top of this improved distribution to further correct the remaining mismatch in a complementary way. 

%% file: jackpot/preliminary.tex
\section{Rejection Sampling
and Optimal Budget Rejection Sampling}
\label{sec:method}
In this section, we briefly introduce optimal budget importance sampling (OBRS) and discuss its theoretical guarantee, followed by our validation via numerical simulation. We further offer detailed analysis and proofs in Appendix~\ref{app:srs-analysis}. 

\subsection{Reduce Distribution Gap via Optimal Budget Importance Sampling}
\label{sec:concept}

We show how to directly modify $p_{\text{inf}}$ to close the distribution gap to $p_{\text{ref}}$ instead of the post-hoc importance sampling corrections.  Let $\boldsymbol{p}$ denote the target distribution we want to align with (e.g., the trained model distribution),  $\boldsymbol{q}$ denote the proposed distribution (e.g., the rollout model distribution), and $\tilde{\boldsymbol{q}}$ denote the \emph{post-rejection distribution} (i.e., the distribution of accepted tokens).

\begin{definition}[Rejection Sampling (RS)]
RS stochastically rejects tokens in the trajectories sampled with $p_{\text{inf}}$ based on the difference between the two distributions. In standard rejection sampling, a token $i$ proposed from $\boldsymbol{q}$ is accepted with probability
\(
\frac{p_i}{\lambda q_i},
\)
where the constant $\lambda$ must satisfy
\(
\lambda \ge \max_i \frac{p_i}{q_i}.
\) Therefore, we have \(
\tilde{q}_i \;\propto\; q_i \cdot \frac{p_i}{\lambda q_i}.
\)
Then
\(
\tilde{q}_i = p_i,
\)
after normalization. Therefore, rejection sampling transforms samples from $\boldsymbol{q}$ into exact samples from the target distribution $\boldsymbol{p}$.
\end{definition}



While standard rejection sampling can, in principle, perfectly align two distributions, its direct application in our setting is impractical. 
In high-dimensional discrete spaces such as LLM vocabularies---often exceeding $100{,}000$ tokens ---even minor local discrepancies between the rollout and target distributions can cause the likelihood ratio $\tfrac{p_i}{q_i}$ to spike for rare tokens. 
This forces the normalizing constant $\lambda$ to become extremely large, which, in turn, drives the acceptance rate to near zero, resulting in almost all proposed tokens being rejected.

To overcome this, we adopt the principled approach of \textbf{Optimal Budgeted Rejection Sampling (OBRS)} \citep{verine2024optimal}. This technique reframes the problem: instead of demanding perfect adherence to the target distribution at the cost of sample efficiency, it seeks the optimal rejection rule that, for a given target acceptance rate (a ``budget"), produces a distribution as close as possible to the target. This is precisely the trade-off our problem requires.

\begin{definition}[Optimal Budget Rejection Sampling (OBRS)~\citep{verine2024optimal}]
Instead of using the dominating constant $\lambda = \max_i \tfrac{p_i}{q_i}$, OBRS selects a smaller user-specified parameter $\lambda > 0$ that reflects the desired rejection budget. For a proposed token $i \sim \boldsymbol{q}$, OBRS accepts it with probability
\(
a_i \;=\; \min\!\left(1,\; \frac{p_i}{\lambda q_i}\right).
\)
The resulting post-rejection distribution is therefore
\(
\tilde{q}_i \;\propto\; q_i \cdot a_i
\)
\end{definition}
While the resulting distribution of OBRS does not generally equal $\boldsymbol{p}$ but is guaranteed to be \emph{closer} to $\boldsymbol{p}$ than the original proposal distribution for any $\lambda$. 
Thus, OBRS provides a controllable trade-off between acceptance rate and distribution alignment, avoiding the vanishing acceptance rates that arise in high-dimensional settings.


\subsection{Numerical Simulation}
\label{sec:numerical}
\begin{figure}[t] 
  \centering 
  \includegraphics[width=\linewidth]{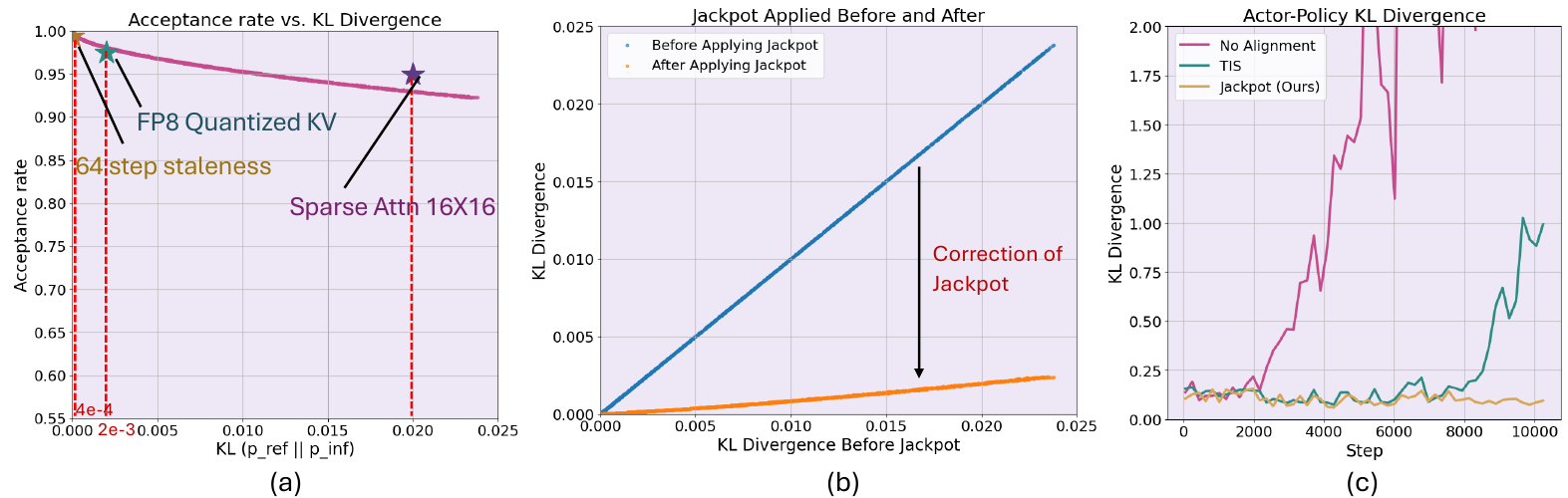} 
  \caption{We conducted numerical experiments in (a) and (b), where we simulate the LLMs' output distribution with randomly generated Dirichlet distribution with controllable noise to attain different levels of KL divergence. (a) plots the simulated \Sys acceptance rate across different pairs of actor-policy distributions. While overal trend sees acceptance rate slowly decreasing as the distributions move further apart, the acceptance rate remains high (\(>90\%\)) throughout the spectrum. For reference, we also marks the actor-policy gap of different common seen off-policy settings. In (b), we show that \Sys significantly shrinks the KL divergence between the target distribution versus the applied distribution in our simulations, sometimes by an order of magnitude. In (c), our proposed method, \Sys (yellow) maintains small KL divergence between actor and policy model probability distribution, while without alginment and TIS both seen KL divergence explosively rise up as the training continues. } 
  \label{fig:obrs}
\end{figure} 

We validate the effectiveness of OBRS via numerical simulation in~\Cref{fig:obrs}. Crucially, this calibration is highly efficient; the acceptance rate remains high even when there is a large initial KL divergence. The impact on distributional alignment is dramatic: a significant reduction in KL divergence is observed with high acceptance rates. By systematically damping the most extreme probability ratios, OBRS produces a distribution that is not only provably closer to the on-policy target but also primed to yield more stable and effective PPO/GRPO policy updates.

 \subsection{Theoretical Guarantees}
\label{sec:theorem}
To justify the use of OBRS within our framework, we establish its fundamental theoretical properties, which formally characterize its ability to reduce distribution mismatch. OBRS possesses proven optimality. We provided our proofs in Appendix~\ref{app:srs-analysis}.

\begin{theorem}[OBRS Improves Distribution Alignment]
Let $\boldsymbol{p}$ be the target distribution and $\boldsymbol{q}$ be the proposal distribution.
For any $\lambda > 0$, define the OBRS acceptance rule
\[
a_i = \min\!\left(1, \frac{p_i}{\lambda q_i}\right),
\qquad 
\tilde{q}_i \;\propto\; q_i a_i .
\]
Then the post-rejection distribution $\tilde{\boldsymbol{q}}$ is strictly closer to $\boldsymbol{p}$ than the original proposal $\boldsymbol{q}$ in the sense that
\[
D_{\mathrm{KL}}(\boldsymbol{p} \,\|\, \tilde{\boldsymbol{q}} )
\; \le\;
D_{\mathrm{KL}}(\boldsymbol{p} \,\|\, \boldsymbol{q} ),
\]
whenever $\lambda < \max_i \frac{p_i}{q_i}$.
\end{theorem}

In other words, OBRS always moves the proposal distribution toward the target distribution under any nontrivial rejection budget.

\begin{theorem}[Optimality of OBRS under a Fixed Acceptance Budget]
For any desired average acceptance rate $\bar{a} \in (0,1]$, there exists a unique scaling factor $\lambda > 0$ such that the OBRS acceptance rule
\[
a_i = \min\!\left(1, \frac{p_i}{\lambda q_i}\right)
\]
achieves the exact acceptance budget:
\[
\sum_i q_i a_i = \bar{a}.
\]
Moreover, among all acceptance rules $a_i \in [0,1]$ that satisfy this constraint, the OBRS rule is the unique minimizer of the divergence to the target distribution:
\[
\tilde{\boldsymbol{q}} 
= \arg\min_{\hat{\boldsymbol{q}}}
D_{\mathrm{KL}}(\boldsymbol{p} \,\|\, \hat{\boldsymbol{q}}),
\quad \text{s.t.}\quad
\hat{q}_i \propto q_i a_i,\;
\sum_i q_i a_i = \bar{a}.
\]
\end{theorem}
Thus, OBRS is provably optimal for aligning $\boldsymbol{q}$ toward $\boldsymbol{p}$ under any specified rejection budget.


The above properties of OBRS are also explored in~\citet{verine2024optimal}. This guarantee ensures we are using the provably best method for trading sample efficiency for distributional accuracy.

%% file: jackpot/method.tex
\section{Jackpot: Design and Methodology} 
\label{sec:jackpot} 
In this section, we present details on the design considerations of \Sys. We show how we apply OBRS in RL, i.e., the token rejection criteria and reweighting procedures in~\Cref{subsec:reject_and_reweighting}, the formulation of \Sys objective in~\Cref{subsec:formulation,subsec:whichpolicy}, and our memory optimization and efficient implementation in~\Cref{subsec:naive-concerns}. 




\subsection{OBRS Procedure} 
\label{subsec:reject_and_reweighting}

To reduce the mismatch between the inference distribution $p_{\text{inf}}$ and the target distribution $p_{\text{target}}$ 
(which may refer to either the reference policy $p_{\text{ref}}$ or the updated policy $p_{\boldsymbol{\theta}_{\text{new}}}$; see~\Cref{subsec:whichpolicy} for details), 
we adopt a rejection rule analogous to~\citet{leviathan2023fast}. 
For a token $x$ sampled from $p_{\text{inf}}$, we accept it with probability
\begin{equation}
\label{eq:obrs_accept}
a(x)
= 
\min\!\left(1, \frac{p_{\text{target}}(x)}{\lambda\, p_{\text{inf}}(x)}\right),
\end{equation}
where $\lambda > 0$ controls the overall rejection budget.

Equation~\eqref{eq:obrs_accept} specifies the conditional probability 
\[
P(x \text{ accepted} \mid x \sim p_{\text{inf}}).
\]
Tokens that are rejected are masked out and excluded from the loss computation and gradient propagation. 
The acceptance rule induces a new (post-rejection) distribution over tokens:
\begin{equation}
\label{eq:obrs_dist}
P_{\text{OBRS}}(x)
=
\frac{
    p_{\text{inf}}(x)\, a(x)
}{
    \sum_{x'} p_{\text{inf}}(x')\, a(x')
}
=
\frac{
    \min\!\left(p_{\text{inf}}(x), \tfrac{p_{\text{target}}(x)}{\lambda}\right)
}{ Z
}.
\end{equation}
where $Z = \sum_{x'} \min\!\left(p_{\text{inf}}(x'), \tfrac{p_{\text{target}}(x')}{\lambda}\right)$.

Importantly, $P_{\text{OBRS}}$ represents the \emph{true} probability distribution governing the accepted samples. 
Therefore, in subsequent training, we reweight all accepted tokens according to $P_{\text{OBRS}}$, ensuring that their contributions to the loss and gradient faithfully reflect their probabilities under the adjusted distribution. We visualize the process in~\Cref{fig:placeholder}.

\subsection{Formulations of the \Sys Objective}
\label{subsec:formulation}

\Sys optimizes three objectives during training:  
(i) an OBRS-adjusted RL loss for the trained model (i.e., policy) model~$\theta$,  
(ii) a standard PPO loss for the rollout model (i.e., actor)~$\omega$, and  
(iii) an on-policy distillation loss that keeps the rollout model aligned with the improving policy.  
We now describe each component.

\paragraph{(1) OBRS-adjusted RL loss for the policy model.}
Following Section~\ref{subsec:reject_and_reweighting}, tokens are sampled from 
the rollout inference distribution $p_{\text{inf}}$, but each sampled token $x$
is accepted or rejected via the OBRS acceptance probability
\[
a(x)
=
\min\!\left(1,\frac{p_{\text{target}}(x)}{\lambda\, p_{\text{inf}}(x)}\right),
\qquad
\mathrm{Mask}(x) \sim \mathrm{Bernoulli}(a(x)).
\]
Accepted tokens ($\mathrm{Mask}(x)=1$) define the OBRS-adjusted distribution
\[
p'_{\text{inf}}(x)
=
\frac{
    p_{\text{inf}}(x)\, a(x)
}{
    \sum_{x'} p_{\text{inf}}(x')\, a(x')
}.
\]

The original PPO objective in Section~\ref{sec:related},
\[
\mathcal{L}^{\text{PPO}}(\theta)
=
\mathbb{E}_{x \sim p_{\text{inf}}}
\left[
    \mathrm{SG}\!\left(
        \boldsymbol{F}\!\left(\tfrac{p_{\text{ref}}(x)}{p_{\text{inf}}(x)}\right)
    \right)
    \min\!\big(
        r_\theta(x)\hat{A}(x),\,
        \operatorname{clip}(r_\theta(x),1-\epsilon,1+\epsilon)\hat{A}(x)
    \big)
\right],
\]
becomes the OBRS-aware PPO objective:
\definecolor{lightblue}{RGB}{210,235,255}
\begin{align}
\label{eq:obrs_ppo}
\mathcal{L}^{\text{PPO-OBRS}}(\theta)
=
\mathbb{E}_{x \sim p_{\text{inf}}}
\Big[
    \colorbox{lightblue}{$\mathrm{Mask}(x)$}\,
    \mathrm{SG}\!\left(
        \boldsymbol{F}\!\left(
            \frac{p_{\text{target}}(x)}{
                \colorbox{lightblue}{$p'_{\text{inf}}(x)$}
            }
        \right)
        \frac{p_{\text{ref}}(x)}{p_{\text{target}}(x)}
    \right)
    \nonumber \\
    \qquad\qquad
    \min\!\big(
        r_\theta(x)\,\hat{A}(x),\,
        \operatorname{clip}(r_\theta(x),1-\epsilon,1+\epsilon)\,\hat{A}(x)
    \big)
\Big],
\end{align}
where rejected samples are removed by $\mathrm{Mask}(x)$, and accepted samples
are reweighted according to $p'_{\text{inf}}(x)$.  The function 
$\boldsymbol{F}$ is the truncated IS correction 
(\Cref{sec:related}), and the choice of $p_{\text{target}}$ is discussed in 
\Cref{subsec:whichpolicy} (Empirically, we can use reference policy $p_{\text{ref}}$ or the updated policy $p_{\boldsymbol{\theta}_{\text{new}}}$ as the $p_{\text{target}}(x)$ for distribution alignment.)  We present the details of this part in~\Cref{alg:jackpot}. 

\paragraph{(2) Standard PPO loss for the rollout model.}
The rollout model~$\omega$ is also optimized using PPO, but without OBRS:
\begin{equation}
\label{eq:ppo_rollout}
\mathcal{L}^{\text{PPO}}(\omega)
=
\mathbb{E}_{x \sim p_{\text{inf}}}
\Big[
    \mathrm{SG}\!\left(
        \boldsymbol{F}\!\left(\tfrac{p_{\omega,\text{ref}}(x)}{p_{\text{inf}}(x)}\right)
    \right)
    \min\!\big(
        r_\omega(x)\,\hat{A}(x),\,
        \operatorname{clip}(r_\omega(x),1-\epsilon,1+\epsilon)\,\hat{A}(x)
    \big)
\Big],
\end{equation}
where $p_{\omega,\text{ref}}$ and $r_\omega$ explicitly denote quantities
computed from the rollout model.

\paragraph{(3) On-policy distillation loss for the rollout model.}
To ensure the rollout model tracks the improving policy, we apply a forward KL
distillation loss using the same sampled trajectories:
\begin{equation}
\label{eq:distill_loss}
\mathcal{L}^{\text{distill}}(\omega)
=
\mathbb{E}_{x \sim p_{\text{inf}}}
\left[
    D_{\mathrm{KL}}\!\left(
        \mathrm{SG}(p_{\theta_{\text{new}}}(x))
        \,\big\|\,
        p_{\omega}(x)
    \right)
\right].
\end{equation}
After optimization, the updated rollout distribution $p_{\omega}$ is used as 
the inference distribution $p_{\text{inf}}$ for the next training iteration, if we do not consider the numerical precision differences between the rollout engine and the training engine~\citep{he2025nondeterminism}.

\paragraph{Joint objective.}
Combining the three components above, the overall \Sys training objective is
\begin{equation}
\label{eq:sys_joint_obj}
\mathcal{L}^{\Sys}(\theta, \omega)
=
\underbrace{\mathcal{L}^{\text{PPO-OBRS}}(\theta)}_{\text{policy RL}}
\;+
\underbrace{\mathcal{L}^{\text{PPO}}(\omega)}_{\text{rollout RL}}
\;+\;
\lambda_{\text{distill}}\,
\underbrace{\mathcal{L}^{\text{distill}}(\omega)}_{\text{on-policy distillation}}
\end{equation}
where  $\lambda_{\text{distill}} \ge 0$ are
hyperparameters that balance the contributions of the rollout RL loss and the
distillation loss.



\begin{figure*}[t] 
    \centering
    \includegraphics[width=\linewidth]{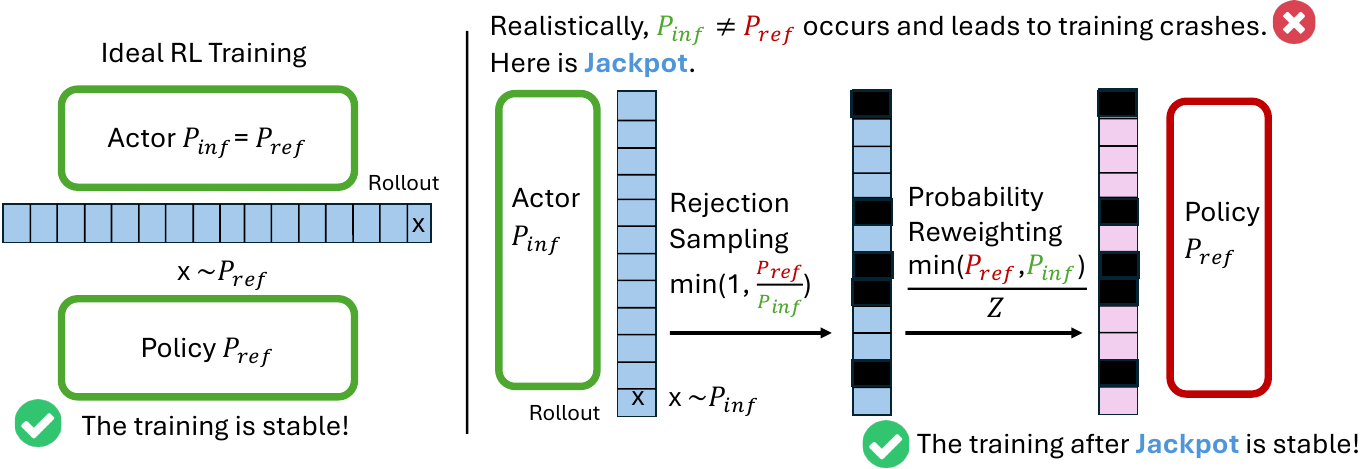}  
    \caption{Illustration of \Sys Pipeline focusing on Optimal Budgeted Rejection Sampling (OBRS) and Reweighting Procedures}
    \label{fig:placeholder}
\end{figure*}

\subsection{Which policy to approximate?} 
\label{subsec:whichpolicy}

OBRS allows us to align the inference distribution $p_{\text{inf}}$ toward 
\emph{any} chosen target distribution $p_{\text{target}}$.  
This flexibility raises a natural question: 
\emph{which policy distribution should we approximate during training?}

A straightforward choice is the reference policy $p_{\text{ref}}$, which is 
consistent with the formulation of standard PPO and provides a stable 
trust-region anchor.  
However, we empirically find that in training regimes where policy staleness 
becomes a major issue, such as large-batch RL or asynchronous 
rollout-update pipelines, the reference policy may drift too far from the 
current updated policy.  
In such cases, aligning $p_{\text{inf}}$ toward the latest policy 
$p_{\text{new}}$ yields better performance, as the trust region defined by 
$p_{\text{ref}}$ is too outdated to offer meaningful guidance.

Therefore, \Sys offers users the flexibility to choose 
$p_{\text{target}} = p_{\text{ref}}$ or $p_{\text{target}} = p_{\text{new}}$,  depending on the requirement in their training setup.

\subsection{Memory Optimization and Efficiency}
\label{subsec:naive-concerns}
Implementing \Sys directly faces a huge challenge of computational feasibility. Note that the weight's normalization constant, $Z$~\Cref{subsec:reject_and_reweighting}, requires a sum over the entire vocabulary ($|\mathcal{V}| > 100,000$), creating a crippling memory bottleneck from storing full logit vectors (\texttt{batch\_size} $\times$ \texttt{seq\_len} $\times$ \texttt{vocab\_size}). This severely restricts batch sizes, directly undermining the efficiency OBRS is intended to provide.  Therefore, transforming this principled approach into a large-scale RL system requires non-trivial engineering: we must introduce mechanisms to both bound the importance weights for stability and develop a computationally efficient, low-bias estimator for the normalization constant. To overcome the computational bottleneck of calculating \(Z\), we employ a top-k approximation and then empirically debias it. We present the details in~\Cref{alg:jackpot}.

\subsubsection{Top-k Approximation} 
\label{subsec:z-approximation} 
The probability mass of language models is typically concentrated in a small subset of the vocabulary. We leverage this property by approximating the sum over \(\mathcal{V}\) with a sum over a much smaller set, \(\mathcal{V}_k\), which contains the most likely tokens from both the inference and current policies. Specifically, let \(\text{top-k}(p)\) be the set of \(k\) tokens with the highest probability under distribution \(p\). We define our approximation set as the union:
$$
\mathcal{V}_k = \text{top-k}(p_{\text{inf}}) \cup \text{top-k}(p_{\boldsymbol{\theta}_{\text{new}}})
$$
The union is crucial because a token might be highly probable under one distribution but not the other, and the \(\min\) function makes these overlapping regions important. The approximate normalization constant, \(Z_{\text{approx}}\), is then: $
Z_{\text{approx}} = \sum_{a' \in \mathcal{V}_k} \min\left(p_{\text{inf}}(a'), \frac{p_{\text{target}}(a')}{\lambda}\right) 
$

\subsubsection{Bias Correction}
\label{subsec:bias_correction}
While efficient, this top-k approximation introduces a systematic bias. Since the terms in the sum are non-negative, omitting tokens from the full vocabulary \(\mathcal{V}\) can only decrease the total sum. Therefore, our approximation is a consistent underestimation of the true value:
$$
\mathbb{E}[Z_{\text{approx}}] \le Z
$$
For k=20, . This bias could systematically alter the scale of the gradients during training. Fortunately, there is an elegant way to correct this. A key property of the framework is that the true normalization constant \(Z\) is exactly equal to the expected acceptance rate, \(\bar{\alpha}\):
$$
\bar{\alpha}
= \sum_{a \in \mathcal{V}} p_{\text{inf}}(a) \cdot 
  \min\left(1, \frac{p_{\text{target}}(a)}{\lambda \cdot p_{\text{inf}}(a)}\right)
= \sum_{a \in \mathcal{V}} 
  \min\left(p_{\text{inf}}(a), \frac{p_{\text{target}}(a)}{\lambda}\right)
= Z. 
$$ 
During the data collection phase (Algorithm~\ref{alg:jackpot}, Phase 1), we can compute an unbiased empirical estimate of \(\bar{\alpha}\) from the observed samples:
$$
\hat{\bar{\alpha}} = \frac{\text{Number of accepted samples}}{\text{Total number of proposed samples}}
$$
This gives us two estimators for \(Z\): the low-variance but biased \(Z_{\text{approx}}\), and the unbiased but higher-variance \(\hat{\bar{\alpha}}\). We can combine them to create a de-biased, low-variance estimator. We compute a batch-wide calibration factor, \(\kappa\), by dividing the empirical acceptance rate by the batch-averaged \(Z_{\text{approx}}\):
$$
\kappa = \frac{\hat{\bar{\alpha}}}{\frac{1}{B}\sum_{i=1}^{B} Z_{\text{approx}}^{(i)}}
$$
where \(B\) is the number of samples in the batch. We then apply this scalar correction to each per-token \(Z_{\text{approx}}\) value used in the loss calculation. This procedure scales our efficient top-k estimate to match the true expected value observed in practice, effectively removing the bias while retaining the computational benefits and lower variance of the top-k approach. 



\subsubsection{Efficiency Analysis}

Despite introducing two additional objectives, the rollout-model RL loss and the
on-policy distillation loss, \Sys remains highly efficient in practice.  
We summarize the key reasons below.

\textbf{(1) Extra losses do not introduce extra rollouts.}
Rollout generation is the dominant computational bottleneck in RL training for
LLMs, typically contributing over 80\% of total runtime.  
Both the rollout-model PPO loss and the distillation loss operate entirely on
the \emph{same trajectories} collected for training the policy model; hence,
no additional rollouts are required.  
Moreover, the rollout model is intentionally chosen to be more efficient than
the policy model (e.g., a smaller variant), further reducing the marginal cost
of updating~$\omega$.  
As a result, the added losses introduce only a small overhead while keeping the
overall training throughput dominated by rollout efficiency.

\textbf{(2) Minimal tensor overhead and no extra probability computation.}
\Sys is lightweight in implementation: it reuses tensors that are already
materialized within the PPO computation graph.  
No additional \texttt{log\_prob} evaluations are needed because both 
$p_{\text{ref}}$ and $p_{\text{new}}$ are produced as part of the standard 
objective~(5).  
\Sys simply reuses these probabilities for OBRS reweighting and alignment.  
Furthermore, \Sys requires no changes to vLLM—no custom operators, kernels, or
special numerical precision assumptions.  
Our implementation runs directly on standard vLLM for rollout without any system-level
modification.

\textbf{(3) No trajectory resampling is required (contrast with speculative decoding or sequence-wise rejection sampling).}
A critical distinction from speculative decoding~\cite{leviathan2023fast} is
that \Sys does \emph{not} resample the remainder of a trajectory once a token
is rejected.  
Speculative decoding discards the entire suffix of a trajectory after the first
mismatch between the draft and target models, requiring additional sampling to
regenerate the remaining tokens.  
In contrast, \Sys simply masks individual rejected tokens according to the OBRS
criterion, while keeping the rest of the trajectory intact.  
This design avoids any resampling of trajectories and significantly reduces
computational overhead.

\begin{algorithm}[h!]
\caption{The Jackpot Algorithm}
\label{alg:jackpot}
\begin{algorithmic}[1]
\REQUIRE Policies: current $p_{\boldsymbol{\text{new}}}$, reference $p_{\text{ref}}$, inference $p_{\text{inf}}$.
\REQUIRE Hyperparameters: OBRS threshold $\lambda$, PPO clip $\epsilon$, Jackpot clips $c_1, c_2$, top-$k$ count.
\STATE \textbf{Convention:} $\text{SG}(\cdot)$ denotes the stop-gradient operation.
\STATE \textbf{Implementation note:} Jackpot only reweights quantities from the \emph{standard}
       rollout and PPO/GRPO forward passes; it does \emph{not} perform extra model forward passes
       or trajectory recomputation.

\STATE \textbf{Phase 1: Efficient Rollout (Standard Generation)}
\STATE Initialize experience buffer $\mathcal{D} \leftarrow \emptyset$.
\FOR{each trajectory sampling step $t$}
    \STATE Single forward pass of $p_{\text{inf}}(\cdot \mid s_t)$, sample $a_t \sim p_{\text{inf}}(\cdot \mid s_t)$.
    \STATE From the same forward, compute and store top-$k$ log-probabilities of $p_{\text{inf}}$:
           $\mathrm{TopK}_{\text{inf}}(s_t)$.
    \STATE Store $(s_t, a_t, p_{\text{inf}}(a_t \mid s_t), \mathrm{TopK}_{\text{inf}}(s_t))$
           (plus rewards, values, etc.) in buffer $\mathcal{D}$.
\ENDFOR 
\STATE Compute advantages $\hat{A}_t$ using collected trajectories.

\STATE \textbf{Phase 2: PPO Update with Jackpot Reweighting}
\FOR{each mini-batch sampled from $\mathcal{D}$}
    \STATE \textbf{// 1. Standard PPO Computation (reused by Jackpot)}
    \STATE Forward pass $p_{\boldsymbol{\text{new}}}$ and $p_{\text{ref}}$ on the mini-batch
           to get logits, $p_{\boldsymbol{\text{new}}}(a_t \mid s_t)$,
           $p_{\text{ref}}(a_t \mid s_t)$, and $\mathrm{TopK}_{\text{new}}(s_t)$.
    \STATE Compute policy ratio:
           $r_t(\boldsymbol{\theta}) =
             \dfrac{p_{\boldsymbol{\text{new}}}(a_t \mid s_t)}
                   {p_{\text{ref}}(a_t \mid s_t)}$.
    \STATE Compute vanilla PPO objective:
           $\mathcal{L}_{\text{PPO}} =
             \min\!\big(r_t(\boldsymbol{\theta})\hat{A}_t,\;
                       \mathrm{clip}(r_t(\boldsymbol{\theta}), 1-\epsilon, 1+\epsilon)\hat{A}_t\big)$.

    \STATE \textbf{// 2. Efficient $Z$-Approximation and Bias Correction (no extra forward passes)}
    \STATE Construct approximation set
           $\mathcal{V}_k = \mathrm{TopK}_{\text{inf}}(s_t) \cup \mathrm{TopK}_{\text{new}}(s_t)$.
    \STATE Compute
           \[
             Z_{\text{approx}} =
               \sum_{x \in \mathcal{V}_k}
                 \min \left( p_{\text{inf}}(x \mid s_t),
                             \frac{p_{\boldsymbol{\text{new}}}(x \mid s_t)}{\lambda} \right).
           \]
    \STATE Estimate correction factor $\kappa$ using the OBRS-based bias-correction
           procedure described in Sec.~\ref{subsec:bias_correction}
           (e.g., from batch-level OBRS statistics).
    \STATE Set corrected normalizer $Z_t \leftarrow \kappa \cdot Z_{\text{approx}}$.

    \STATE \textbf{// 3. Jackpot Weight Calculation}
    \STATE OBRS weight:
           $w_{\text{OBRS}} =
             Z_t \cdot
             \max\!\left(\lambda,
                        \frac{p_{\boldsymbol{\text{new}}}(a_t \mid s_t)}
                             {p_{\text{inf}}(a_t \mid s_t)}\right)$.
    \STATE $\rho_{\text{jackpot}} =
           \min(w_{\text{OBRS}}, c_1) \cdot
           \min\!\left(
             \frac{p_{\text{ref}}(a_t \mid s_t)}{p_{\boldsymbol{\text{new}}}(a_t \mid s_t)},
             c_2
           \right)$.
    
    \STATE \textbf{// 4. Apply Weight to Loss}
    \STATE $\mathcal{L}_{\text{final}} = \text{SG}(\rho_{\text{jackpot}}) \cdot \mathcal{L}_{\text{PPO}}$.
    
    \STATE Update policy parameters $\boldsymbol{\text{new}}$ using gradient of $-\mathcal{L}_{\text{final}}$.
\ENDFOR
\end{algorithmic}
\end{algorithm} 

%% file: jackpot/empiricalstudies.tex
\section{Empirical Analysis} 
\label{sec:newexp} 

\begin{figure}
    \centering
    \includegraphics[width=\linewidth]{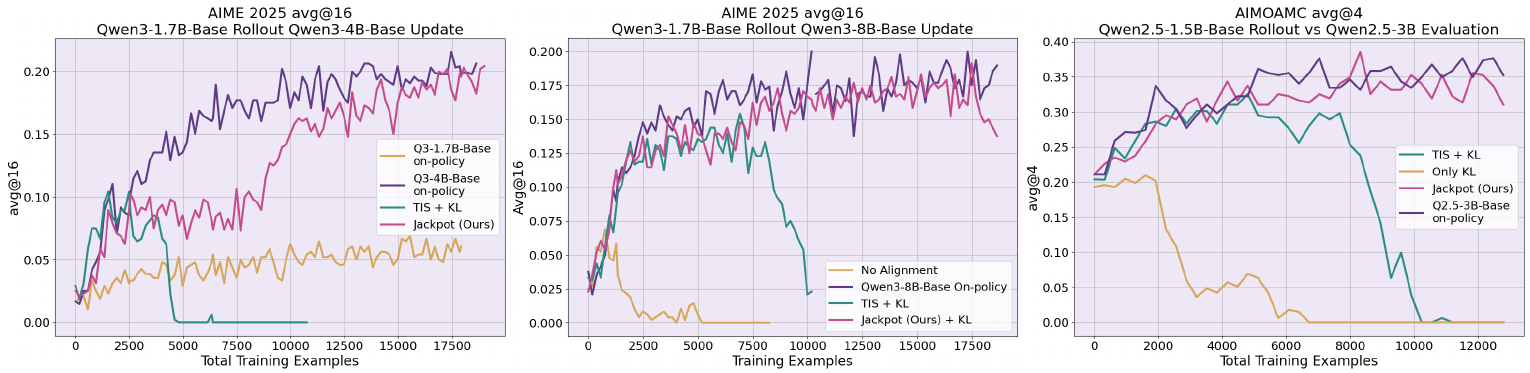} 
    \caption{\sys enables probability distribution alignment beyond existing methods. On the extreme two model joint training setting, with \sys, the smaller and weaker model is able to rollout trajectories which are used by the bigger stronger models for computing its training. We show that prior TIS methods, even added the KL, consistently suffers from unstable training across three different settings. Qwen2.5 series 1.5B and 3B, Qwen3 1.7B and 4B, and Qwen3 1.7B and 8B base models. In contrast, \sys leads to comparable performance with the large model on-policy performance.}
    \label{fig:expri} 
\end{figure} 

\Sys can be layered on top of truncated importance sampling (TIS) and enables
training in regimes where actor–policy mismatch is extremely large.  
In this section, we consider the most challenging configuration: training a 
strong, expensive policy model using a completely separate, smaller, and more 
efficient actor model for rollouts.  
Unlike scenarios where the actor model is merely stale or an approximate 
quantized version of the policy, employing a fully separate model introduces 
KL divergences that are orders of magnitude larger than those observed in 
conventional asynchronous-training setups.  
At the same time, such decoupling offers the potential for substantial rollout 
cost reductions, making this setting both practically important and 
algorithmically demanding.

\begin{table*}[t]
\centering
\caption{
Evaluation results across datasets for different actor–policy training configurations.
Mean@4 and Pass@4 are computed using four independent samples.
}
\small
\begin{adjustbox}{max width=\linewidth}
\begin{tabular}{lcccccccc}
\toprule
\textbf{Models} 
& \textbf{GSM8K} 
& \textbf{MATH-500} 
& \textbf{AMC22/23} 
& \textbf{AMC12} 
& \textbf{AIME24} 
& \textbf{AIME24} 
& \textbf{AIME25} 
& \textbf{AIME25} \\
& \textbf{Mean@4} 
& \textbf{Mean@4} 
& \textbf{Mean@4} 
& \textbf{Mean@4} 
& \textbf{Mean@4} 
& \textbf{Pass@4}
& \textbf{Mean@4} 
& \textbf{Pass@4} \\
\midrule

\multicolumn{9}{c}{
\textit{$\text{Qwen2.5-1.5B} \to \text{Qwen2.5-3B}$ (14k, 64, MATH-8K)}
} \\
\midrule
Q2.5-3B On-policy      & 0.8500 & 0.6390 & 0.3765 & 0.2611 &   --   &   --   &   --   &   --   \\
Only Reverse KL        & 0.7417 & 0.4535 & 0.2093 & 0.1407 &   --   &   --   &   --   &   --   \\
TIS + Reverse KL       & 0.8250 & 0.6045 & 0.3253 & 0.2444 &   --   &   --   &   --   &   --   \\
\rowcolor{cyan!10} 
\sys (Ours)         & 0.8428 & 0.6275 & 0.3855 & 0.2778 &   --   &   --   &   --   &   --   \\
\midrule

\multicolumn{9}{c}{
\textit{$\text{Qwen3-1.7B} \to \text{Qwen3-4B}$ (20k, 64, DeepScaleR)}
} \\
\midrule
Q3-1.7B On-policy      & 0.8380 & 0.6590 & 0.3524 & 0.2500 & 0.1000 & 0.1741 & 0.0680 & 0.1336 \\
Q3-4B On-policy        & 0.9256 & 0.8082 & 0.5813 & 0.5166 & 0.2500 & 0.3514 & 0.2156 & 0.2966 \\
TIS + Reverse KL       & 0.9121 & 0.7365 & 0.4639 & 0.3277 & 0.1333 & 0.2031 & 0.1041 & 0.1912 \\
\rowcolor{cyan!10} 
\sys (Ours)         & 0.9215 & 0.8052 & 0.5949 & 0.5388 & 0.2350 & 0.3364 & 0.2083 & 0.2778 \\
\midrule

\multicolumn{9}{c}{
\textit{$\text{Qwen3-1.7B} \to \text{Qwen3-8B}$ (15k, 64, DeepScaleR)}
} \\
\midrule
Q3-8B On-policy        & 0.9329 & 0.7950 & 0.6114 & 0.5333 & 0.2437 & 0.3385 & 0.1687 & 0.2616 \\
TIS + Reverse KL       & 0.9361 & 0.7645 & 0.5662 & 0.3722 & 0.1770 & 0.2741 & 0.1541 & 0.2223 \\
\rowcolor{cyan!10} 
\sys (Ours)         & 0.9357 & 0.8265 & 0.6204 & 0.5444 & 0.2500 & 0.3657 & 0.1916 & 0.2769 \\
\bottomrule
\end{tabular}
\end{adjustbox}

\end{table*}



\paragraph{Setup.}
We evaluate \Sys in the context of training LLMs on mathematical reasoning tasks
using the GRPO algorithm~\cite{guo2025deepseek}.  
Our experiments span three strong–weak actor–policy model pairs across two model
families and two datasets:  
(1) Qwen2.5-1.5B-Base (actor) with Qwen2.5-3B-Base (policy) trained on the MATH
dataset~\cite{hendrycks2021measuringmathematicalproblemsolving};  
(2) Qwen3-1.7B-Base with Qwen3-4B-Base trained on the DeepScaleR dataset
\cite{deepscaler2025}; and  
(3) Qwen3-1.7B-Base with Qwen3-8B-Base also trained on DeepScaleR.  
The maximum generation length is set to 8192 tokens for all experiments.

\paragraph{Results.}
Naively applying existing methods to this decoupled-actor setting leads to 
highly unstable training, even when truncated importance sampling (TIS) is used 
to constrain the distribution gap between the reference and inference models—training 
still collapses.  
In contrast, \Sys maintains stable learning for a substantially extended number 
of update steps, demonstrating stability up to 300 steps across the evaluated 
model pairs.  
These results highlight the effectiveness of OBRS-based alignment in mitigating 
distribution mismatch and improving robustness in decoupled RL training.

We show the results in the Figure. We show that the large model performance can indeed be stably trained for a large number of steps and cover a large number of examples. Surprisingly, for the Qwen3-4B-Base model, our joint training even matches the 4B model performance when training itself. \Sys consistently outperforms TIS even with reverse KL added. The above results marks the potential for significant RL training cost savings. 

\section{Ablation Studies} 
\label{sec:analysis_ablation}

\begin{figure}
    \centering
    \includegraphics[width=\linewidth]{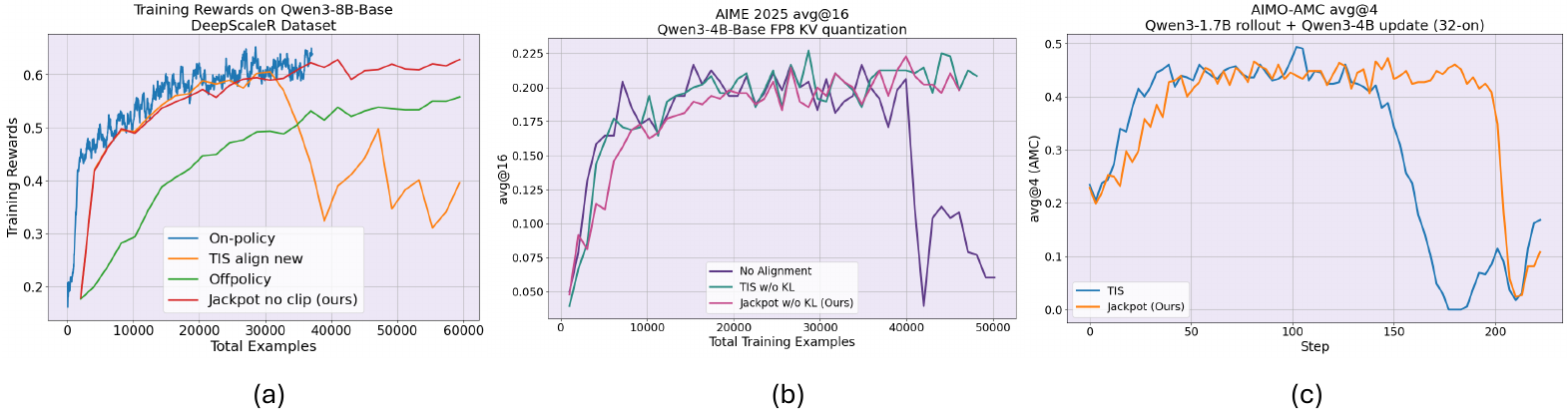} 
    \caption{(a) \sys enables removal of clipping from stale RL training. (b) \sys isn't showing improvement nor harm when actor-policy distributions are relatively close and can be sufficiently corrected by TIS. (c) When removing KL, \Sys consistently sustains longer than TIS counterparts.}  
    \label{fig:ablation} 
\end{figure} 

\subsection{When \sys isn't Effective} 
We find that \sys offers little improvement over standard methods when the distribution shift is inherently small or already well-controlled by existing techniques.

RL algorithms like PPO and GRPO employ clipping mechanisms to enforce trust regions, explicitly constraining the deviation of the current policy from the old (rollout) policy. While this design enhances stability, it inherently limits the magnitude of the update per step. Consequently, the distribution shift between the rollout policy and the current policy remains small. In these regimes, the distribution correction provided by \sys becomes less critical, resulting in performance on par with, rather than superior to, standard baselines. 

Similarly, while FP8 KV quantization during rollout introduces a larger distribution gap between the training and rollout frameworks by increasing the "off-policy" nature of the task, standard techniques like TIS effectively mitigate this instability (See Figure~\ref{fig:ablation}). Because TIS is sufficient to restore performance levels comparable to unquantized baselines, the additional distribution correction offered by \sys yields diminishing returns in this context. Detailed evaluations are demonstrated in Table~\ref{tab:no-benefit}.

\begin{table}
\centering
\caption{\sys provides little additional benefit when distribution shifts are already controlled by PPO clipping or TIS under FP8 KV quantization.}
\begin{adjustbox}{max width=\linewidth}
\begin{tabular}{lrrrrrrrr}
\toprule
\multicolumn{1}{c}{} &
\multicolumn{1}{c}{GSM8K} &
\multicolumn{1}{c}{MATH-500} &
\multicolumn{1}{c}{AMC22 \& 23} &
\multicolumn{1}{c}{AMC12 2024} &
\multicolumn{2}{c}{AIME24} &
\multicolumn{2}{c}{AIME25} \\
\cmidrule(lr){2-2}\cmidrule(lr){3-3}\cmidrule(lr){4-4}\cmidrule(lr){5-5}\cmidrule(lr){6-7}\cmidrule(lr){8-9}
\multicolumn{1}{c}{Models / Methods} &
\multicolumn{1}{c}{Mean@4} &
\multicolumn{1}{c}{Mean@4} &
\multicolumn{1}{c}{Mean@4} &
\multicolumn{1}{c}{Mean@4} &
\multicolumn{1}{c}{Mean@16} & \multicolumn{1}{c}{Pass@16} &
\multicolumn{1}{c}{Mean@16} & \multicolumn{1}{c}{Pass@16} \\
\midrule 
\rowcolor{gray!15} 
\multicolumn{9}{c}{\textbf{Qwen3-4B-Base on DeepScaleR-Preview Dataset} (rollout batch size = 2048; train batch size = 32; 64$\times$)} \\ 
\midrule 
\rowcolor{gray!15} 
Off Policy          & 92.10 & 81.55 & 60.54 & 56.11 & 27.50 & 38.36 & 23.12 & 31.36 \\
\rowcolor{cyan!10} 
\sys (Ours)      & 92.53 & 81.95 & 59.94 & 56.11 & 27.71 & 37.77 & 22.70 & 31.22 \\
\midrule 
\rowcolor{gray!15} 
\multicolumn{9}{c}{\textbf{Qwen3-4B-Base on DeepScaleR-Preview Dataset} (rollout batch size = 128; train batch size = 64; FP8 KV)} \\ 
\midrule 
TIS          & 92.68 & 83.65 & 60.84 & 53.89 & 25.83 & 36.95 & 22.70 & 29.2 \\
\rowcolor{cyan!10} 
\sys (Ours)      & 91.83 & 81.30 & 62.35 & 50.56 & 24.79 & 35.13 & 22.29 & 29.51 \\
\bottomrule 
\label{tab:no-benefit} 
\end{tabular}
\end{adjustbox} 
\end{table}

\subsection{What \sys enables} 
\label{jptenabls} 

In contrast, \sys provides substantial benefits in regimes with large policy shifts, enabling both stable training and faster convergence. See Figure~\ref{fig:ablation}.

To evaluate the benefits of \sys, we investigate a training regime where the protective constraints of PPO clipping are removed. While clipping ensures stability by limiting the policy update magnitude, it simultaneously throttles the learning speed. By relaxing these constraints, we simulate a scenario with significant distribution shifts: a setting where robust distribution correction is extremely necessary. We compare the performance of three configurations over the first 30,000 training examples: a \textbf{no-clip} baseline (off-policy without PPO clipping), a \textbf{vanilla off-policy} baseline, and our proposed \textbf{\sys} method (remove PPO clipping).

Our empirical results of Qwen3-4B-Base with rollout batch size $4096$ ($128\times$ staleness) in Table~\ref{tab:when_benefit} highlight a critical trade-off between stability and convergence speed that standard methods fail to navigate: 1. In the absence of trust-region clipping, the \textbf{no-clip} run crashes mid-training. 2. The \textbf{vanilla off-policy} baseline maintains stability throughout the training run. While it demonstrates constant, monotonic improvement in performance, it suffers from slow convergence speeds. Without an effective mechanism to correct for the distribution lag between the behavior and target policies, the model learns inefficiently. While the training does not crash with rollout batch size $2048$ ($64\times$ staleness) for either model, we still observe clear benefits of \sys in both convergence speed and final accuracy. Because all methods are evaluated under the same fixed budgets of $30$K and $50$K training examples, these higher accuracies directly indicate faster convergence in the large-batch regime. 

\sys successfully combines the benefits of the conflicting baselines without their respective drawbacks. Unlike the \textbf{no-clip} setting, \sys maintains robust stability and completes the training without crashing, effectively managing the large distribution shifts introduced by unclipped updates. Furthermore, it significantly outperforms the vanilla off-policy baseline in terms of convergence speed. By providing a more accurate distribution correction, \sys enables the model to take larger, stable optimization steps, accelerating learning in regimes where standard trust-region constraints are significantly loosened or removed.

\begin{table}
\centering
\caption{Ablation under large distribution shift (no PPO clipping): \sys enables stable and fast convergence across benchmarks. We report the best test accuracy for the first 30K training examples on Qwen3-4B-Base and 50K on Qwen3-8B-Base.}
\begin{adjustbox}{max width=\linewidth}
\begin{tabular}{lrrrrrrrr}
\toprule
\multicolumn{1}{c}{} &
\multicolumn{1}{c}{GSM8K} &
\multicolumn{1}{c}{MATH-500} &
\multicolumn{1}{c}{AMC22 \& 23} &
\multicolumn{1}{c}{AMC12 2024} &
\multicolumn{2}{c}{AIME24} &
\multicolumn{2}{c}{AIME25} \\
\cmidrule(lr){2-2}\cmidrule(lr){3-3}\cmidrule(lr){4-4}\cmidrule(lr){5-5}\cmidrule(lr){6-7}\cmidrule(lr){8-9}
\multicolumn{1}{c}{Models / Methods} &
\multicolumn{1}{c}{Mean@4} &
\multicolumn{1}{c}{Mean@4} &
\multicolumn{1}{c}{Mean@4} &
\multicolumn{1}{c}{Mean@4} &
\multicolumn{1}{c}{Mean@16} & \multicolumn{1}{c}{Pass@16} &
\multicolumn{1}{c}{Mean@16} & \multicolumn{1}{c}{Pass@16} \\
\midrule 
\rowcolor{gray!15} 
\multicolumn{9}{c}{\textbf{Qwen3-4B-Base on DeepScaleR-Preview Dataset} (rollout batch size = 2048; train batch size = 32; 64$\times$)} \\ 
\midrule 
\rowcolor{gray!15} 
On Policy           & 92.19 & 81.55 & 58.43 & 51.11 & 23.12 & 33.13 & 22.91 & 30.95 \\
Off Policy          & 88.04 & 71.15 & 39.15 & 29.44 & 13.96 & 23.03 & 11.04 & 18.61 \\
No-Clip    & \textbf{92.76} & 79.50 & \textbf{57.22} & 43.33 & 18.75 & 26.03 & 17.71 & 24.61 \\
\rowcolor{cyan!10} 
\sys (Ours)      & 92.24 & \textbf{80.05} & 53.92 & \textbf{50.00} & \textbf{20.63} & \textbf{29.48} & \textbf{18.13} & \textbf{23.63} \\
\midrule 
\rowcolor{gray!15} 
\multicolumn{9}{c}{\textbf{Qwen3-4B-Base on DeepScaleR-Preview Dataset} (rollout batch size = 4096; train batch size = 32; 128$\times$)} \\ 
\midrule 
Off Policy          & 79.70 & 60.20 & 33.00 & 24.44 & 8.00 & 15.73 & 5.00 & 11.00 \\
No-Clip    & 20.70 & 19.10 & 7.80 & 5.00 & 1.00 & 4.00 & 1.00 & 2.00 \\
\rowcolor{cyan!10} 
\sys (Ours)      & \textbf{92.00} & \textbf{80.00} & \textbf{51.20} & \textbf{47.22} & \textbf{19.16} & \textbf{24.58} & \textbf{18.52} & \textbf{25.08} \\
\midrule 
\rowcolor{gray!15} 
\multicolumn{9}{c}{\textbf{Qwen3-8B-Base on DeepScaleR-Preview Dataset} (rollout batch size = 2048; train batch size = 32; 64$\times$)} \\
\midrule 
\rowcolor{gray!15} 
On Policy           & 94.24 & 93.99 & 68.95 & 54.44 & 28.95 & 37.89 & 22.50 & 28.54 \\
Off Policy          & 91.05 & 77.15 & 50.60 & 40.00 & 18.54 & 28.67 & 14.16 & 21.98 \\
No-Clip    & 93.85 & 82.55 & 60.54 & 48.33 & 24.58 & 35.06 & 20.00 & 22.90 \\ 
\rowcolor{cyan!10} 
\sys (Ours)      & \textbf{94.01} & \textbf{83.05} & \textbf{63.55} & \textbf{54.44} & \textbf{26.87} & \textbf{36.23} & \textbf{20.41} & \textbf{26.57} \\
\bottomrule 
\label{tab:when_benefit} 
\end{tabular}
\end{adjustbox} 
\end{table} 



%% file: jackpot/conclusion.tex

\section{Conclusion}

We presented \sys, a framework that leverages Optimal Budget 
Rejection Sampling to reduce the distribution discrepancy between the rollout 
model and the policy model in reinforcement learning for large language models.  
Through a principled OBRS procedure, a unified training objective that couples 
policy and rollout updates, and an efficient system implementation enabled by 
\topk-based probability estimation and batch-level bias correction, \sys moves 
the RL training pipeline a step closer toward fully decoupling rollout 
generation from policy optimization.  
This work highlights a practical direction for enabling more flexible and 
efficient training regimes for large language models.

\section{Limitations} 

\begin{wrapfigure}{r}{0.32\linewidth}
    \includegraphics[width=\linewidth]{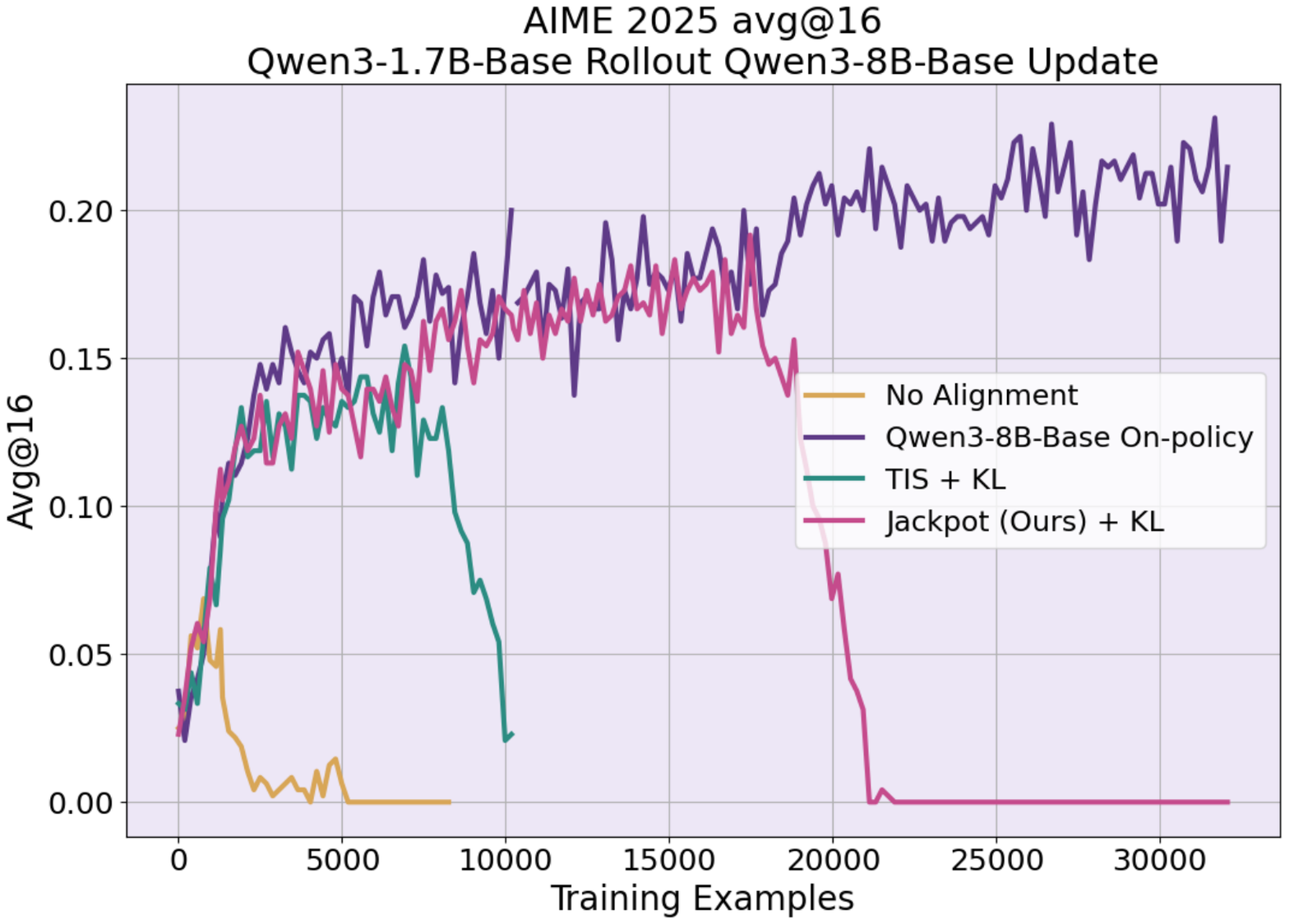} 
    \caption{\textbf{Limitation:} Even though \Sys opens up the possibility of using rollout of a much smaller model to train a bigger and more expensive model, it still leads to collapse eventually.} 
\end{wrapfigure} 

As shown in Section~\ref{sec:analysis_ablation}, when the distribution shift is
already small or adequately controlled by existing techniques, \sys yields only
minor improvements over standard baselines.  
Moreover, although \sys reduces the mismatch between the rollout model and the
policy model, using a fully separate, smaller actor model to train a large and
expensive policy does not completely eliminate the distribution gap or the
resulting training instability.  
In our experiments, we observe that training may \textbf{still} diverge after extended
optimization (e.g., beyond 300 update steps), and \sys has not yet been 
validated on larger-scale models such as 32B variants. 

Achieving a fully robust decoupling between rollout generation and policy
optimization likely requires additional mechanisms beyond OBRS.  
A promising direction is to introduce a closed-loop control scheme that adapts
the relative strength of the distillation loss and RL loss based on real-time
measurements of the KL divergence between the rollout and policy models.

%% file: jackpot/appendix.tex
\section{Analysis of OBRS}
\label{app:srs-analysis} 

This appendix provides the theoretical foundation for OBRS. We first formally define the post-rejection distribution that results from our method. We then prove two key results:
\begin{enumerate}
    \item \textbf{Optimality:} For any desired sample efficiency (i.e., acceptance rate), OBRS is the unique optimal rejection mechanism that produces a post-rejection distribution closest to the target \(p_{\boldsymbol{\theta}_{\text{new}}}\) in terms of KL divergence (Theorem~\ref{thm:budget-opt}).
    \item \textbf{Monotonic Improvement:} The post-rejection distribution monotonically approaches the target distribution as the scaling factor \(C\) increases. This guarantees that for any \(C>0\), OBRS reduces the KL divergence compared to using the proposal \(p_{\text{inf}}\) directly. (Proposition~\ref{prop:kl-monotonicity}).
\end{enumerate}

For notational clarity, we consider the distributions over tokens for a single, fixed prompt and omit the explicit conditioning. Let \(p_t \equiv p_{\boldsymbol{\theta}_{\text{new}}}\) denote the target distribution and \(p_p \equiv p_{\text{inf}}\) denote the proposal distribution.

\begin{algorithm}[h!]
\caption{Implementation of Optimal Budgeted Rejection Sampling}
\label{alg:srs}
\begin{algorithmic}[1]
\REQUIRE Proposal distribution \(p_{\text{inf}}\), Target distribution \(p_{\boldsymbol{\theta}_{\text{new}}}\)
\REQUIRE Scaling factor \(C > 0\), Number of samples to accept \(N\)
\ENSURE Set of accepted samples \(S_{\text{kept}}\)

\STATE Initialize \(S_{\text{kept}} \leftarrow \emptyset\)
\WHILE{\(|S_{\text{kept}}| < N\)}
    \STATE Sample a token \(a \sim p_{\text{inf}}(\cdot)\)
    \STATE Calculate acceptance probability \(\alpha \leftarrow \min\left(1, \frac{p_{\boldsymbol{\theta}_{\text{new}}}(a)}{C \cdot p_{\text{inf}}(a)}\right)\)
    \IF{\(U(0, 1) < \alpha\)}
        \STATE Add \(a\) to \(S_{\text{kept}}\)
    \ENDIF
\ENDWHILE
\STATE \RETURN \(S_{\text{kept}}\)
\end{algorithmic}
\end{algorithm}

\subsection{The Post-Rejection Distribution}

Recall from Definition~1 that OBRS accepts a token \(a \sim p_p(a)\) with probability \(\alpha_C(a) = \min\left(1, \frac{p_t(a)}{C \cdot p_p(a)}\right)\). The unnormalized probability of sampling and keeping a token \(a\) is \(p_p(a) \cdot \alpha_C(a)\), which simplifies to \(\min\{p_p(a), p_t(a)/C\}\).

The overall probability of accepting any token, which we denote as the acceptance rate \(Z_C\), is the sum over all possible tokens:
\[
Z_C = \sum_{a \in \mathcal{A}} \min\left\{p_p(a), \frac{p_t(a)}{C}\right\}.
\]
The distribution of the tokens that are kept, which we call the post-rejection distribution \(p_{\text{kept},C}\), is therefore:
\[
p_{\text{kept},C}(a) = \frac{\min\{p_p(a), p_t(a)/C\}}{Z_C}.
\]
Special cases clarify the role of \(C\): as \(C \to 0\), \(Z_C \to 1\) and \(p_{\text{kept},C} \to p_p\) (all tokens are kept). As \(C \to \infty\), \(Z_C \to 0\) and \(p_{\text{kept},C} \to p_t\) (perfect alignment with vanishing throughput). Standard rejection sampling is the special case where \(C \geq \max_a (p_t(a)/p_p(a))\).

\subsection{Optimality for a Fixed Acceptance Budget}
\label{app:budget-optimality}

We first establish that OBRS is not merely a heuristic but is the provably optimal strategy for minimizing distributional error given a fixed efficiency budget.

\begin{theorem}[Budgeted Optimal Acceptance]
\label{thm:budget-opt}
Fix a target acceptance rate (budget) \(z \in (0, 1]\). Among all possible token-wise acceptance rules \(\alpha: \mathcal{A} \to [0, 1]\) that satisfy the budget constraint \(\mathbb{E}_{a \sim p_p}[\alpha(a)] = \sum_a p_p(a) \alpha(a) = z\), the rule that generates a post-rejection distribution \(p_{\text{kept}}(a) \propto p_p(a) \alpha(a)\) minimizing the Kullback-Leibler (KL) divergence \(\text{KL}(p_t \,||\, p_{\text{kept}})\) is uniquely given by the OBRS rule:
\[
\alpha^\star(a) = \min\left(1, \frac{p_t(a)}{\lambda \cdot p_p(a)}\right)
\]
for some constant \(\lambda > 0\) (equivalent to \(C\)) whose value is determined by the budget \(z\).
\end{theorem}

\begin{proof}
The post-rejection distribution is \(p_{\text{kept}}(a) = p_p(a)\alpha(a)/z\). The KL divergence is:
\begin{align*}
\text{KL}(p_t \,||\, p_{\text{kept}}) &= \sum_a p_t(a) \log\frac{p_t(a)}{p_{\text{kept}}(a)} \\
&= \sum_a p_t(a) \log\frac{p_t(a)}{p_p(a)\alpha(a)/z} \\
&= \underbrace{\sum_a p_t(a) \log\frac{p_t(a)}{p_p(a)} + \log z}_{\text{constant w.r.t. } \alpha} - \sum_a p_t(a)\log\alpha(a).
\end{align*}
Minimizing \(\text{KL}(p_t \,||\, p_{\text{kept}})\) is therefore equivalent to maximizing \(\sum_a p_t(a)\log\alpha(a)\) subject to the constraints:
\begin{enumerate}
    \item \(\sum_a p_p(a) \alpha(a) = z\) (budget constraint)
    \item \(0 \le \alpha(a) \le 1\) for all \(a \in \mathcal{A}\) (valid probability constraint)
\end{enumerate}
This is a convex optimization problem. The Lagrangian is:
\[
\mathcal{L} = -\sum_a p_t(a)\log\alpha(a) + \lambda\left(\sum_a p_p(a)\alpha(a) - z\right) + \sum_a \mu_a ( \alpha(a) - 1 ) - \sum_a \nu_a \alpha(a)
\]
where \(\lambda, \mu_a, \nu_a\) are the KKT multipliers. From the stationarity condition \(\frac{\partial \mathcal{L}}{\partial \alpha(a)} = 0\), we get \(-\frac{p_t(a)}{\alpha(a)} + \lambda p_p(a) + \mu_a - \nu_a = 0\).
The complementary slackness conditions imply that if \(0 < \alpha(a) < 1\), then \(\mu_a = \nu_a = 0\), which gives \(\alpha(a) = \frac{p_t(a)}{\lambda p_p(a)}\). If \(\alpha(a) = 1\), then \(\nu_a = 0\), which requires \(\lambda p_p(a) \geq p_t(a)\). If \(\alpha(a) = 0\), this form is not well-defined, but the logic holds. Combining these cases, the optimal rule is to cap the acceptance ratio at 1:
\[ \alpha^\star(a) = \min\left(1, \frac{p_t(a)}{\lambda p_p(a)}\right). \]
The Lagrange multiplier \(\lambda\) is chosen to meet the budget constraint \(\sum_a p_p(a)\alpha^\star(a)=z\). Uniqueness follows from the strict concavity of the log objective function.
\end{proof}

\subsection{Guaranteed KL Divergence Reduction}

Next, we show that our method provides a guaranteed improvement over the proposal distribution \(p_p\) and that this improvement is monotonic in the control parameter \(C\).

\begin{proposition}[Monotonic KL Contraction]
\label{prop:kl-monotonicity}
The function \(G(C) = \text{KL}(p_t \,||\, p_{\text{kept},C})\) is non-increasing for \(C \in (0, \infty)\). It is strictly decreasing wherever the set of tokens \(\{a \mid p_t(a) < C \cdot p_p(a)\}\) has non-zero probability mass under \(p_t\).
\end{proposition}

\begin{proof}
Let \(\rho(a) = p_t(a)/p_p(a)\). We partition the vocabulary \(\mathcal{A}\) into two sets: \(A_C = \{a \mid \rho(a) > C\}\) and \(B_C = \{a \mid \rho(a) \le C\}\). On any open interval of \(C\) where this partition is constant, we can write the KL divergence \(G(C) = -\sum_a p_t(a)\log p_{\text{kept},C}(a)\) as:
\[
G(C) = -\sum_{a \in A_C} p_t(a)\log\frac{p_p(a)}{Z_C} - \sum_{a \in B_C} p_t(a)\log\frac{p_t(a)/C}{Z_C}
\]
Differentiating with respect to \(C\) (and noting that only \(Z_C\) depends on \(C\)), we get \(\frac{dG}{dC} = -\frac{d}{dC} \sum_a p_t(a) (-\log Z_C) = \frac{1}{Z_C} \frac{dZ_C}{dC}\).
The acceptance rate is \(Z_C = \sum_{a \in A_C} p_p(a) + \frac{1}{C}\sum_{a \in B_C} p_t(a)\). Its derivative is:
\[
\frac{dZ_C}{dC} = -\frac{1}{C^2}\sum_{a \in B_C} p_t(a).
\]
Therefore, \(G'(C) = \frac{1}{Z_C} \left( -\frac{1}{C^2}\sum_{a \in B_C} p_t(a) \right) \le 0\), since all terms are non-negative. The derivative is strictly negative if \(\sum_{a \in B_C} p_t(a) > 0\). As \(G(C)\) is continuous and piecewise differentiable with a non-positive derivative, it is non-increasing everywhere.
\end{proof}

\begin{corollary}[Strict Improvement over Proposal]
For any \(C > 0\), OBRS produces a distribution \(p_{\text{kept},C}\) that is strictly closer to the target distribution \(p_t\) than the original proposal distribution \(p_p\), i.e.,
\[
\text{KL}(p_t \,||\, p_{\text{kept},C}) < \text{KL}(p_t \,||\, p_p),
\]
unless \(p_p = p_t\) or \(C \leq min(\frac{p_t}{p_p})\), in which case the KL divergences are both zero .
\end{corollary}
\begin{proof}
From Proposition~\ref{prop:kl-monotonicity}, we know that \(\text{KL}(p_t \,||\, p_{\text{kept},C})\) is non-increasing in \(C\). In the limit as \(C \to 0\), the acceptance probability \(\alpha_C(a) \to 1\) for all \(a\), meaning \(p_{\text{kept},C} \to p_p\). Therefore, \(\lim_{C\to 0} \text{KL}(p_t \,||\, p_{\text{kept},C}) = \text{KL}(p_t \,||\, p_p)\).
For any \(C>0\), as long as \(p_p \neq p_t\) and \(C > min(\frac{p_t}{p_p})\), there must exist some tokens for which \(p_t(a) < C \cdot p_p(a)\) or \(p_t(a) > C \cdot p_p(a)\), ensuring the condition for a strictly decreasing KL divergence is met over some interval \((0, C]\). Thus, \(G(C) < G(\epsilon)\) for some small \(\epsilon > 0\), which implies \(\text{KL}(p_t \,||\, p_{\text{kept},C}) < \text{KL}(p_t \,||\, p_p)\).
\end{proof}

\subsection{Practical guidance}
\begin{itemize}[leftmargin=1.5em]
\item \textbf{Setting $C$.} $C{=}1$ is a robust default: it contracts the per‑prompt KL (strictly unless $p_\text{inf}{=}p_\text{tr}$), keeps acceptance high ($Z_1$), and is $O(1)$ per token. Larger $C$ pushes $q_C$ closer to $p_\text{tr}$ but reduces throughput ($Z_C\!\le\!1/C$); use it only if variance or bias considerations demand stronger alignment.
\item \textbf{Compatibility.} The rule uses only per‑token log‑probabilities already computed by PPO/GRPO, so it introduces no new estimators and preserves gradient flow exactly as described in Algorithm \ref{alg:jackpot}.
\end{itemize} 

\section{Details of the Experiments} 
\subsection{TIS Adjustment Explanation}
\label{sec:tisadjustment}
However, the delay in updates results in staleness, which we magnify for evaluating our method. 

In this section, we push the above scenario to its limit by asking the inference batch size to be 64x and 128x the training batch size. Concretely, we choose to use a training batch size of 32. We train models on the DeepScalerR-Preview dataset \cite{deepscaler2025}, which contains 40k challenging competition math problems. We select Qwen3-4B-Base and Qwen3-8B-Base models \cite{yang2025qwen3technicalreport} to run RL on. An important baseline to our method is the Truncated Importance Sampling (TIS) as in \cite{wu2025llamarldistributedasynchronousreinforcement}. However, in the original technique is proposed only for the approximate models in the inference server. Thus, in the following loss form. 
\[
E_{a\sim\pi(\theta_{\text{old}})}[\frac{\pi(\theta_{\text{ref}})}{\pi(\theta_{\text{old,inf}})}\nabla_{\theta}\text{clip}(\frac{\pi(\theta_{\text{new}})}{\pi(\theta_{\text{ref}})}\hat{A})] 
\]
The weights update frequency is close to on-policy settings, but the gap between the approximate and efficient inference model and the training weights is the primary goal to solve. However, the original formula cannot easily be adapted for our extremely large batch setting, as there is no term regularizing the difference between $\pi_{\theta_{\text{new}}}$ and $\pi_{\theta_{\text{old}}}$. We found that a very simple trick results in a very strong baseline on top of the TIS method, that is we write it this way. 
\[
E_{a\sim\pi(\theta_{\text{old}})}[\frac{\pi(\theta_{\text{new,detached}})}{\pi(\theta_{\text{old,inf}})}\nabla_{\theta}clip(\frac{\pi(\theta_{\text{new}})}{\pi(\theta_{\text{new,detached}})}\hat{A})] 
\] where we use the detached most recent model output distribution as the term in the importance sampling. However, the consequence is also very clear, the internal clip around ratio is now 'short-circuited' or no longer useful. Nevertheless, the setting produces very strong convergence is correction over the off-policy baseline. We call it TIS-adjusted. 

To fairly compare against the baseline, we also modify our training loss as follows, effectively also 'short-circuiting' the internal ratio clip. The only difference between our setting and theirs is that we use our proposed sampling method to regularize the $\pi_{\theta_{\text{old,inf}}}$. 
\[
E_{a\sim\pi(\theta_{\text{old}})}[\frac{\pi(\theta_{\text{new,detached}})}{\pi(\theta_{\text{old,inf}})^{*\text{new}}}\nabla_{\theta}\text{clip}(\frac{\pi(\theta_{\text{new}})}{\pi(\theta_{\text{new,detached}})}\hat{A})] 
\]. 

\subsection{Full Details of Extreme Size Batch Size Experiments} 

In the following section, we provide more details about the extreme staleness experiments listed in Section~\ref{jptenabls}. The full details are presented in Table~\ref{staleone} and Table~\ref{staletwo}. 

\begin{table}[t]
\centering
\caption{Evaluation scores across benchmarks (GSM8K, MATH-500, AMC22 \& AMC23).}
\begin{adjustbox}{max width=\linewidth}
\begin{tabular}{l *{6}{r}}
\toprule
\multicolumn{1}{c}{} &
\multicolumn{2}{c}{GSM8K} &
\multicolumn{2}{c}{MATH-500} &
\multicolumn{2}{c}{AMC22 \& AMC23} \\
\cmidrule(lr){2-3}\cmidrule(lr){4-5}\cmidrule(lr){6-7}
\multicolumn{1}{c}{Models / Methods} &
\multicolumn{1}{c}{Mean@4} & \multicolumn{1}{c}{Pass@4} &
\multicolumn{1}{c}{Mean@4} & \multicolumn{1}{c}{Pass@4} &
\multicolumn{1}{c}{Mean@4} & \multicolumn{1}{c}{Pass@4} \\
\midrule
\multicolumn{7}{c}{\textbf{Qwen3-4B-Base on DeepScaler} (rollout batch size = 2048; train batch size = 32; 64$\times$)}
\\
\midrule 
On Policy                 & 92.19  & 95.988 & 81.55 & 88.00 & 58.43  & 71.76  \\
Off Policy                & 88.04  & 95.03  & 71.15 & 82.24 & 39.15  & 55.57  \\
TIS                       & 89.67  & 95.53  & 72.00 & 80.96 & 42.77  & 56.71  \\
TIS with Adjustment       & \textbf{92.76}  & \textbf{96.09} & 79.50 & 85.81 & \textbf{57.22} & \textbf{66.61}  \\
Jackpot (Ours)            & 92.24  & 95.891 & \textbf{80.05} & \textbf{85.89} & 53.916 & 65.034 \\
\addlinespace
\multicolumn{7}{c}{\textbf{Qwen3-4B-Base on DeepScaler} (rollout batch size = 4096; train batch size = 32; 128$\times$)} \\
\midrule 
Off Policy                & 79.70  & 92.96  & 60.20 & 76.60 & 33.00  & 48.446 \\
TIS with Adjustment       & 20.697 & 43.00  & 19.10 & 37.751& 7.80   & 17.83  \\
Jackpot (Ours)            & \textbf{92.00}  & \textbf{95.00}  & \textbf{80.00} & \textbf{85.50} & \textbf{51.20}  & \textbf{60.35}  \\
\addlinespace
\multicolumn{7}{c}{\textbf{Qwen3-8B-Base on DeepScaler} (rollout batch size = 2048; train batch size = 32; 64$\times$)} \\
\midrule 
On Policy                 & 94.238 & 96.78  & 93.99 & 96.65 & 28.95  & 37.89  \\
Off Policy                & 91.05  & 95.62  & 77.15 & 84.90 & 50.60  & 65.20  \\
TS with Adjustment        & 93.85  & 96.58  & 82.55 & 88.38 & 60.54  & 72.79  \\
Jackpot (Ours)            & \textbf{94.01}  & \textbf{96.63}  & \textbf{83.05} & \textbf{88.76} & \textbf{63.55}  & \textbf{74.12}  \\
\bottomrule 
\label{staleone} 
\end{tabular}
\end{adjustbox}
\end{table}

\begin{table}
\centering
\caption{Evaluation scores across benchmarks (AMC12 2024, AIME24, AIME25).}
\begin{adjustbox}{max width=\linewidth}
\begin{tabular}{l *{6}{r}}
\toprule
\multicolumn{1}{c}{} &
\multicolumn{2}{c}{AMC12 2024} &
\multicolumn{2}{c}{AIME24} &
\multicolumn{2}{c}{AIME25} \\
\cmidrule(lr){2-3}\cmidrule(lr){4-5}\cmidrule(lr){6-7}
\multicolumn{1}{c}{Models / Methods} &
\multicolumn{1}{c}{Mean@4} & \multicolumn{1}{c}{Pass@4} &
\multicolumn{1}{c}{Mean@16} & \multicolumn{1}{c}{Pass@16} &
\multicolumn{1}{c}{Mean@16} & \multicolumn{1}{c}{Pass@16} \\
\midrule
\multicolumn{7}{c}{\textbf{Qwen3-4B-Base on DeepScaler} (rollout batch size = 2048; train batch size = 32; 64$\times$)} \\
\midrule 
On Policy                 & 51.11  & 65.45  & 23.12  & 33.13  & 22.91  & 30.95  \\
Off Policy                & 29.44  & 41.802 & 13.958 & 23.03  & 11.042 & 18.607 \\
TIS                       & 31.667 & 50.844 & 11.875 & 17.561 & 11.25  & 17.278 \\
TIS with Adjustment       & 43.33  & 60.67  & 18.75  & 26.03  & 17.708 & 24.607 \\
Jackpot (Ours)            & \textbf{50.00}  & \textbf{63.00}  & \textbf{20.625} & \textbf{29.484} & \textbf{18.125} & \textbf{23.627} \\
\addlinespace
\multicolumn{7}{c}{\textbf{Qwen3-4B-Base on DeepScaler} (rollout batch size = 4096; train batch size = 32; 128$\times$)} \\
\midrule 
Off Policy                & 24.44  & 38.41  & 8.00  & 15.73  & 5.00  & 11.00  \\
TIS with Adjustment       & 5.00  & 11.00  & 1.00  & 4.00  & 1.00  & 2.00  \\
Jackpot (Ours)            & \textbf{47.22}  & \textbf{57.99} & \textbf{19.16}  & \textbf{24.58}  & \textbf{18.52}  & \textbf{25.078} \\
\addlinespace
\multicolumn{7}{c}{\textbf{Qwen3-8B-Base on DeepScaler} (rollout batch size = 2048; train batch size = 32; 64$\times$)} \\
\midrule 
On Policy                 & 54.44  & 68.47  & 28.95  & 37.89  & 22.50  & 28.542 \\
Off Policy                & 40.00  & 54.75  & 18.54  & 28.67  & 14.16  & 21.98  \\
TS with Adjustment        & 48.33  & 59.78  & 24.58  & 35.06  & 20.00  & 22.90  \\
Jackpot (Ours)            & \textbf{54.44}  & \textbf{66.09}  & \textbf{26.87}  & \textbf{36.23}  & \textbf{20.41}  & \textbf{26.57}  \\
\bottomrule 
\label{staletwo} 
\end{tabular}
\end{adjustbox}
\end{table} 

\section{Ablation Studies on Components of \Sys, Threshold, and Top-K} 
\label{ablationstudies} 

\subsection{Components' Contribution to \Sys} 
Rejection bridges the gap between the rollout and the current policy (see FP8 on-policy training without explicit jackpot reweighting; this is essentially the vanilla PPO loss with the probability fixed, and the OBRS distribution in this case matches the reference policy exactly). Because the rollout--training distribution gap is now removed, stability is significantly better than the vanilla baseline: even without TIS, training does not crash.

However, without correct importance sampling (see the huge-staleness training regime, where the importance distribution no longer matches $P_{\text{ref}}$ but instead matches the current policy), training will eventually collapse.

On the other hand, jackpot reweighting does \emph{not} solve the rollout--training gap (because under the FP8 setting it falls back to vanilla on-policy training, which again leads to a crash). But in the huge-staleness regime, where the rollout--training gap is not the main issue, jackpot reweighting combined with OBRS is effective: it performs correct importance sampling, tracks the proper target distribution, and keeps the overall training procedure stable and efficient. 


\begin{table}[h!]
\centering
\caption{FP8 on-policy training (no staleness). Best scores before crash for the vanilla baseline.}
\begin{tabular}{lcccc}
\toprule
\textbf{Experiments} & \textbf{AIME24} & \textbf{AMC} & \textbf{MATH500} & \textbf{GSM8K} \\
\midrule
Vanilla (best before crash) & 23.958 & 57.530  & 80.900  & 92.665 \\
Masking-only                 & 25.625 & 62.048 & 83.700 & 92.835 \\
Masking \& reweighting       & 26.667 & 62.651 & 82.450 & 92.305 \\
\bottomrule
\end{tabular}
\end{table}


\begin{table}[h!]
\centering
\caption{BF16 training with rollout staleness (64/2048). Best scores before crash for masking-only.}
\begin{tabular}{lcccc}
\toprule
\textbf{Experiments} & \textbf{AIME24} & \textbf{AMC} & \textbf{MATH500} & \textbf{GSM8K} \\
\midrule
Masking-only (best before crash) & 19.167 & 49.699 & 78.750 & 91.793 \\
Masking \& reweighting           & 25.625 & 63.855 & 83.800  & 92.400   \\
\bottomrule
\end{tabular}
\end{table} 

\subsection{Thresholds, C1, and C2} 

Our method involves three hyperparameters: $C_1$, $C_2$, and the rejection threshold $\lambda$. All of them are straightforward to set, and the technique is robust across a wide range of choices.

\paragraph{$C_1$.}
We follow standard truncated importance sampling (TIS) choices. Empirically, selecting $C_1 \in [2, 10]$ consistently works well, and the method is not sensitive within this interval.

\paragraph{$C_2$ (upper bound for $p_{\theta_{\mathrm{ref}}} / p_{\theta_{\mathrm{new}}}$).}
This parameter has no practical effect on performance. We set $C_2$ slightly larger than $1 + \varepsilon_{\mathrm{high}}$ (e.g., $1.28$ for DAPO), where any ratio clipped by $C_2$ would already be clipped by the PPO trust region. Thus, $C_2$ mainly serves as a conceptual safeguard for ratio stability.

\paragraph{Rejection threshold $\lambda$.}
Our method performs well across all experiments with a default setting of $\lambda = 1.0$, and we recommend choosing $\lambda$ close to this value. Increasing $\lambda$ makes the kept-token distribution closer to the target policy but increases the rejection rate. If $\lambda > 1$, it begins rejecting tokens even when the policy and inference distributions are perfectly aligned, causing overly conservative updates. The default value $c = 1.0$ guarantees full acceptance in the matched-distribution case and already reduces KL substantially while keeping a high acceptance rate.

\paragraph{Summary.}
Jackpot is stable and easy to configure: $C_1$ is robust within the typical TIS range $2\!-\!10$, $C_2$ has no practical impact once chosen above $1 + \varepsilon_{\mathrm{high}}$, and $c = 1.0$ serves as a reliable default. 

\begin{table}[h!]
\centering
\caption{Effect of $C_1$ on benchmark performance. \textbf{Experiment Setup:} Model: Qwen3-4B-Base, C2 = 3.0, threshold $c = 1.0$, response limit: 8k, mini-batch/train-batch: 64/2048, PPO clip: 0.4/0.7, 100k examples. Numbers are pass@1 accuracy.}
\begin{tabular}{lcccc}
\toprule
\textbf{Hyperparameters} & \textbf{AIME24} & \textbf{AMC} & \textbf{MATH500} & \textbf{GSM8K} \\
\midrule
$C_1 = 2$ & 26.875 & 63.855 & 82.800  & 92.267 \\
$C_1 = 3$ & 25.625 & 63.855 & 83.800  & 92.703 \\
$C_1 = 4$ & 26.042 & 65.060 & 83.500  & 92.684 \\
$C_1 = 8$ & 26.875 & 63.253 & 83.100  & 92.437 \\
\bottomrule
\end{tabular}
\end{table} 

\begin{table}[h!]
\centering
\caption{Effect of rejection threshold $c$ on benchmark performance. \textbf{Experiment Setup:} Model: Qwen3-4B-Base (target) / Qwen3-1.7B-Base (rollout); Generation length limit: 8K; Training examples: 9K. Numbers are pass@1 accuracy.}
\begin{tabular}{lcccc}
\toprule
\textbf{Threshold $c$} & \textbf{AIME24} & \textbf{AMC} & \textbf{MATH500} & \textbf{GSM8K} \\
\midrule
0.8 & 14.7 & 49.4 & 74.5 & 92.0 \\[3pt]
0.9 & 12.5 & 47.0 & 74.6 & 91.8 \\[3pt]
1.0 & 14.7 & 48.5 & 74.6 & 92.2 \\[3pt]
1.1 & 12.3 & 47.4 & 74.1 & 92.0 \\[3pt]
1.2 & 13.5 & 45.8 & 74.1 & 91.9 \\
\bottomrule
\end{tabular}
\vspace{-2mm}
\end{table}

\vspace{-2mm}
\subsection{Choice of Top-$K$ for $Z$ Approximation} 
\vspace{-2mm}
In Section~\ref{subsec:naive-concerns} we approximate the OBRS normalization constant $Z$ by summing over the union of the top-$k$ tokens under $p_{\mathrm{inf}}$ and $p_{\theta_{\text{new}}}$, yielding $Z_{\text{approx}} \le Z$. Increasing $k$ strictly improves this approximation but also increases the number of logits that must be materialized and stored. In the main experiments we therefore fix $k=20$, and here we justify this choice empirically.

We first study the direct effect of $k$ on the quality of the $Z$ estimator. For the extreme off-policy configuration, we log (i) the fraction of the true normalization captured by the top-$k$ estimator, $Z_{\text{approx}}/Z$, and (ii) the calibration factor $\kappa$ defined in Section~\ref{subsec:bias_correction}, using $k\in\{10,20,40\}$. As expected, larger $k$ improves both quantities, but with rapidly diminishing returns. Even for the smallest value $k=10$ we already capture at least $87\%$ of $Z$ at the very beginning of training, and more than $99.98\%$ once the policy has warmed up within few steps; using $k=40$ increases the captured mass only slightly (from roughly $91\%$ initially to about $100.0\%$ in the steady state). Since the union of top-$20$ tokens is always a superset of the union of top-$10$ tokens, these diagnostics imply that $k=20$ already yields an almost exact estimator of $Z$ while keeping the additional overhead modest.

We next evaluate how $k$ affects downstream performance. Table~\ref{tab:topk_ablation} reports pass@1 accuracy on our four math benchmarks for Jackpot with $k \in \{10,20,40\}$, keeping all other hyperparameters fixed. The differences across choices of $k$ are small and non-monotonic: $k=20$ performs slightly better on AMC and MATH500, while AIME24 and GSM8K show no consistent trend. Overall, the variation is comparable to run-to-run noise, and there is no evidence that pushing $k$ beyond $20$ systematically improves task performance.
\begin{table}[ht!]
\centering
\caption{Effect of top-$k$ on benchmark performance. Numbers are pass@1 accuracy.}
\label{tab:topk_ablation}
\begin{tabular}{lcccc}
\toprule
\textbf{Hyperparameters} & \textbf{AIME24} & \textbf{AMC} & \textbf{MATH500} & \textbf{GSM8K} \\
\midrule
$k=10$ & 28.958 & 61.446 & 82.650 & 92.608 \\
$k=20$ & 25.625 & 63.855 & 83.800 & 92.703 \\
$k=40$ & 27.083 & 61.446 & 83.200 & 92.418 \\
\bottomrule
\end{tabular}
\end{table}
Taken together, these results show that larger $k$ does improve the $Z$ estimator, but the gains become marginal once $k$ reaches $20$. Since the computational overhead of our method scales roughly linearly with $k$, we adopt $k=20$ as a practical default: it provides an accurate, well-calibrated estimate of $Z$ with negligible additional cost (less than $3\%$ overhead in our setup), and larger values of $k$ do not yield measurable benefits in our experiments.

%% file: main.bbl
\begin{thebibliography}{36}
\providecommand{\natexlab}[1]{#1}
\providecommand{\url}[1]{\texttt{#1}}
\expandafter\ifx\csname urlstyle\endcsname\relax
  \providecommand{\doi}[1]{doi: #1}\else
  \providecommand{\doi}{doi: \begingroup \urlstyle{rm}\Url}\fi

\bibitem[Ahmadian et~al.(2024)Ahmadian, Cremer, Gall{\'e}, Fadaee, Kreutzer, Pietquin, {\"U}st{\"u}n, and Hooker]{ahmadian2024back}
Arash Ahmadian, Chris Cremer, Matthias Gall{\'e}, Marzieh Fadaee, Julia Kreutzer, Olivier Pietquin, Ahmet {\"U}st{\"u}n, and Sara Hooker.
\newblock Back to basics: Revisiting reinforce style optimization for learning from human feedback in llms.
\newblock \emph{arXiv preprint arXiv:2402.14740}, 2024.

\bibitem[Azerbayev et~al.(2023)Azerbayev, Schoelkopf, Paster, Santos, McAleer, Jiang, Deng, Biderman, and Welleck]{azerbayev2023llemma}
Zhangir Azerbayev, Hailey Schoelkopf, Keiran Paster, Marco~Dos Santos, Stephen McAleer, Albert~Q Jiang, Jia Deng, Stella Biderman, and Sean Welleck.
\newblock Llemma: An open language model for mathematics.
\newblock \emph{arXiv preprint arXiv:2310.10631}, 2023.

\bibitem[Brown et~al.(2024)Brown, Juravsky, Ehrlich, Clark, Le, R{\'e}, and Mirhoseini]{brown2024large}
Bradley Brown, Jordan Juravsky, Ryan Ehrlich, Ronald Clark, Quoc~V Le, Christopher R{\'e}, and Azalia Mirhoseini.
\newblock Large language monkeys: Scaling inference compute with repeated sampling.
\newblock \emph{arXiv preprint arXiv:2407.21787}, 2024.

\bibitem[Espeholt et~al.(2018)Espeholt, Soyer, Munos, Simonyan, Mnih, Ward, Doron, Firoiu, Harley, Dunning, Legg, and Kavukcuoglu]{espeholt2018impalascalabledistributeddeeprl}
Lasse Espeholt, Hubert Soyer, Remi Munos, Karen Simonyan, Volodymir Mnih, Tom Ward, Yotam Doron, Vlad Firoiu, Tim Harley, Iain Dunning, Shane Legg, and Koray Kavukcuoglu.
\newblock Impala: Scalable distributed deep-rl with importance weighted actor-learner architectures, 2018.
\newblock \url{https://arxiv.org/abs/1802.01561}.

\bibitem[Fu et~al.(2025)Fu, Gao, Shen, Zhu, Mei, He, Xu, Wei, Mei, Wang, et~al.]{fu2025areal}
Wei Fu, Jiaxuan Gao, Xujie Shen, Chen Zhu, Zhiyu Mei, Chuyi He, Shusheng Xu, Guo Wei, Jun Mei, Jiashu Wang, et~al.
\newblock Areal: A large-scale asynchronous reinforcement learning system for language reasoning.
\newblock \emph{arXiv preprint arXiv:2505.24298}, 2025.

\bibitem[Guo et~al.(2025)Guo, Yang, Zhang, Song, Zhang, Xu, Zhu, Ma, Wang, Bi, et~al.]{guo2025deepseek}
Daya Guo, Dejian Yang, Haowei Zhang, Junxiao Song, Ruoyu Zhang, Runxin Xu, Qihao Zhu, Shirong Ma, Peiyi Wang, Xiao Bi, et~al.
\newblock Deepseek-r1: Incentivizing reasoning capability in llms via reinforcement learning.
\newblock \emph{arXiv preprint arXiv:2501.12948}, 2025.

\bibitem[He and Lab(2025)]{he2025nondeterminism}
Horace He and Thinking~Machines Lab.
\newblock Defeating nondeterminism in llm inference.
\newblock \emph{Thinking Machines Lab: Connectionism}, 2025.
\newblock \doi{10.64434/tml.20250910}.
\newblock https://thinkingmachines.ai/blog/defeating-nondeterminism-in-llm-inference/.

\bibitem[Hendrycks et~al.(2021)Hendrycks, Burns, Kadavath, Arora, Basart, Tang, Song, and Steinhardt]{hendrycks2021measuringmathematicalproblemsolving}
Dan Hendrycks, Collin Burns, Saurav Kadavath, Akul Arora, Steven Basart, Eric Tang, Dawn Song, and Jacob Steinhardt.
\newblock Measuring mathematical problem solving with the math dataset, 2021.
\newblock \url{https://arxiv.org/abs/2103.03874}.

\bibitem[Hu et~al.(2024)Hu, Wu, Zhu, Xianyu, Wang, Zhang, and Cao]{hu2024openrlhf}
Jian Hu, Xibin Wu, Zilin Zhu, Xianyu, Weixun Wang, Dehao Zhang, and Yu~Cao.
\newblock Openrlhf: An easy-to-use, scalable and high-performance rlhf framework.
\newblock \emph{arXiv preprint arXiv:2405.11143}, 2024.

\bibitem[Jimenez et~al.(2023)Jimenez, Yang, Wettig, Yao, Pei, Press, and Narasimhan]{jimenez2023swe}
Carlos~E Jimenez, John Yang, Alexander Wettig, Shunyu Yao, Kexin Pei, Ofir Press, and Karthik Narasimhan.
\newblock Swe-bench: Can language models resolve real-world github issues?
\newblock \emph{arXiv preprint arXiv:2310.06770}, 2023.

\bibitem[Jin et~al.(2025)Jin, Zeng, Yue, Yoon, Arik, Wang, Zamani, and Han]{jin2025search}
Bowen Jin, Hansi Zeng, Zhenrui Yue, Jinsung Yoon, Sercan Arik, Dong Wang, Hamed Zamani, and Jiawei Han.
\newblock Search-r1: Training llms to reason and leverage search engines with reinforcement learning.
\newblock \emph{arXiv preprint arXiv:2503.09516}, 2025.

\bibitem[Kiran et~al.(2021)Kiran, Sobh, Talpaert, Mannion, Al~Sallab, Yogamani, and P{\'e}rez]{kiran2021deep}
B~Ravi Kiran, Ibrahim Sobh, Victor Talpaert, Patrick Mannion, Ahmad~A Al~Sallab, Senthil Yogamani, and Patrick P{\'e}rez.
\newblock Deep reinforcement learning for autonomous driving: A survey.
\newblock \emph{IEEE transactions on intelligent transportation systems}, 23\penalty0 (6):\penalty0 4909--4926, 2021.

\bibitem[Leviathan et~al.(2023)Leviathan, Kalman, and Matias]{leviathan2023fast}
Yaniv Leviathan, Matan Kalman, and Yossi Matias.
\newblock Fast inference from transformers via speculative decoding.
\newblock In \emph{International Conference on Machine Learning}, pages 19274--19286. PMLR, 2023.

\bibitem[Li et~al.(2023)Li, Xu, Zhang, Lin, Yu, Sun, and Luo]{li2023remax}
Ziniu Li, Tian Xu, Yushun Zhang, Zhihang Lin, Yang Yu, Ruoyu Sun, and Zhi-Quan Luo.
\newblock Remax: A simple, effective, and efficient reinforcement learning method for aligning large language models.
\newblock \emph{arXiv preprint arXiv:2310.10505}, 2023.

\bibitem[Liu et~al.(2025{\natexlab{a}})Liu, Li, Fu, Wang, Liu, and Shen]{liu-li-2025-rl-collapse}
Jiacai Liu, Yingru Li, Yuqian Fu, Jiawei Wang, Qian Liu, and Yu~Shen.
\newblock When speed kills stability: Demystifying {RL} collapse from the training-inference mismatch, September 2025{\natexlab{a}}.
\newblock \url{https://richardli.xyz/rl-collapse}.

\bibitem[Liu et~al.(2025{\natexlab{b}})Liu, Yao, Zhang, Dong, Shang, and Gao]{liu2025flashrl}
Liyuan Liu, Feng Yao, Dinghuai Zhang, Chengyu Dong, Jingbo Shang, and Jianfeng Gao.
\newblock Flashrl: 8bit rollouts, full power rl, August 2025{\natexlab{b}}.
\newblock \url{https://fengyao.notion.site/flash-rl}.

\bibitem[Liu et~al.(2023)Liu, Yu, Zhang, Xu, Lei, Lai, Gu, Ding, Men, Yang, Zhang, Deng, Zeng, Du, Zhang, Shen, Zhang, Su, Sun, Huang, Dong, and Tang]{liu2023agentbench}
Xiao Liu, Hao Yu, Hanchen Zhang, Yifan Xu, Xuanyu Lei, Hanyu Lai, Yu~Gu, Hangliang Ding, Kaiwen Men, Kejuan Yang, Shudan Zhang, Xiang Deng, Aohan Zeng, Zhengxiao Du, Chenhui Zhang, Sheng Shen, Tianjun Zhang, Yu~Su, Huan Sun, Minlie Huang, Yuxiao Dong, and Jie Tang.
\newblock Agentbench: Evaluating llms as agents.
\newblock \emph{arXiv preprint arXiv: 2308.03688}, 2023.

\bibitem[Luo et~al.(2025)Luo, Tan, Wong, Shi, Tang, Roongta, Cai, Luo, Li, Popa, and Stoica]{deepscaler2025}
Michael Luo, Sijun Tan, Justin Wong, Xiaoxiang Shi, William~Y. Tang, Manan Roongta, Colin Cai, Jeffrey Luo, Li~Erran Li, Raluca~Ada Popa, and Ion Stoica.
\newblock Deepscaler: Surpassing o1-preview with a 1.5b model by scaling rl, 2025.
\newblock Notion Blog.

\bibitem[Meng et~al.(2024)Meng, Xia, and Chen]{meng2024simpo}
Yu~Meng, Mengzhou Xia, and Danqi Chen.
\newblock Simpo: Simple preference optimization with a reference-free reward.
\newblock \emph{Advances in Neural Information Processing Systems}, 37:\penalty0 124198--124235, 2024.

\bibitem[Ouyang et~al.(2025)Ouyang, Guo, Arora, Zhang, Hu, R{\'e}, and Mirhoseini]{ouyang2025kernelbench}
Anne Ouyang, Simon Guo, Simran Arora, Alex~L Zhang, William Hu, Christopher R{\'e}, and Azalia Mirhoseini.
\newblock Kernelbench: Can llms write efficient gpu kernels?
\newblock \emph{arXiv preprint arXiv:2502.10517}, 2025.

\bibitem[Qin et~al.(2025)Qin, He, Huang, Zhang, Zhao, Pang, Xu, Shan, Wu, and Zhang]{qin2025seer}
Ruoyu Qin, Weiran He, Weixiao Huang, Yangkun Zhang, Yikai Zhao, Bo~Pang, Xinran Xu, Yingdi Shan, Yongwei Wu, and Mingxing Zhang.
\newblock Seer: Online context learning for fast synchronous llm reinforcement learning.
\newblock \emph{arXiv preprint arXiv:2511.14617}, 2025.

\bibitem[Rafailov et~al.(2023)Rafailov, Sharma, Mitchell, Manning, Ermon, and Finn]{rafailov2023direct}
Rafael Rafailov, Archit Sharma, Eric Mitchell, Christopher~D Manning, Stefano Ermon, and Chelsea Finn.
\newblock Direct preference optimization: Your language model is secretly a reward model.
\newblock \emph{Advances in neural information processing systems}, 36:\penalty0 53728--53741, 2023.

\bibitem[Sadhukhan et~al.(2025)Sadhukhan, Chen, Zheng, Zhou, Strubell, and Chen]{sadhukhan2025kinetics}
Ranajoy Sadhukhan, Zhuoming Chen, Haizhong Zheng, Yang Zhou, Emma Strubell, and Beidi Chen.
\newblock Kinetics: Rethinking test-time scaling laws.
\newblock \emph{arXiv preprint arXiv:2506.05333}, 2025.

\bibitem[Schulman et~al.(2017)Schulman, Wolski, Dhariwal, Radford, and Klimov]{schulman2017proximalpolicyoptimizationalgorithms}
John Schulman, Filip Wolski, Prafulla Dhariwal, Alec Radford, and Oleg Klimov.
\newblock Proximal policy optimization algorithms, 2017.
\newblock \url{https://arxiv.org/abs/1707.06347}.

\bibitem[Shao et~al.(2024)Shao, Wang, Zhu, Xu, Song, Bi, Zhang, Zhang, Li, Wu, et~al.]{shao2024deepseekmath}
Zhihong Shao, Peiyi Wang, Qihao Zhu, Runxin Xu, Junxiao Song, Xiao Bi, Haowei Zhang, Mingchuan Zhang, YK~Li, Yang Wu, et~al.
\newblock Deepseekmath: Pushing the limits of mathematical reasoning in open language models.
\newblock \emph{arXiv preprint arXiv:2402.03300}, 2024.

\bibitem[Sheng et~al.(2025)Sheng, Zhang, Ye, Wu, Zhang, Zhang, Peng, Lin, and Wu]{sheng2025hybridflow}
Guangming Sheng, Chi Zhang, Zilingfeng Ye, Xibin Wu, Wang Zhang, Ru~Zhang, Yanghua Peng, Haibin Lin, and Chuan Wu.
\newblock Hybridflow: A flexible and efficient rlhf framework.
\newblock In \emph{Proceedings of the Twentieth European Conference on Computer Systems}, pages 1279--1297, 2025.

\bibitem[Team et~al.(2025)Team, Shen, Li, Hu, Jing, Chen, Huang, Zhang, Yang, Lin, et~al.]{team2025every}
Ling Team, Anqi Shen, Baihui Li, Bin Hu, Bin Jing, Cai Chen, Chao Huang, Chao Zhang, Chaokun Yang, Cheng Lin, et~al.
\newblock Every step evolves: Scaling reinforcement learning for trillion-scale thinking model.
\newblock \emph{arXiv preprint arXiv:2510.18855}, 2025.

\bibitem[Together-AI()]{Atlas}
Together-AI.
\newblock Adaptive-learning speculator system (atlas): A new paradigm in llm inference via runtime-learning accelerators.
\newblock \url{https://www.together.ai/blog/adaptive-learning-speculator-system-atlas}.

\bibitem[Verine et~al.(2024)Verine, Pydi, Negrevergne, and Chevaleyre]{verine2024optimal}
Alexandre Verine, Muni~Sreenivas Pydi, Benjamin Negrevergne, and Yann Chevaleyre.
\newblock Optimal budgeted rejection sampling for generative models.
\newblock In \emph{International Conference on Artificial Intelligence and Statistics}, pages 3367--3375. PMLR, 2024.

\bibitem[von Werra et~al.(2020)von Werra, Belkada, Tunstall, Beeching, Thrush, Lambert, Huang, Rasul, and Gallouédec]{vonwerra2022trl}
Leandro von Werra, Younes Belkada, Lewis Tunstall, Edward Beeching, Tristan Thrush, Nathan Lambert, Shengyi Huang, Kashif Rasul, and Quentin Gallouédec.
\newblock Trl: Transformer reinforcement learning.
\newblock \url{https://github.com/huggingface/trl}, 2020.

\bibitem[Wu et~al.(2025{\natexlab{a}})Wu, Wang, Tang, Ding, Helenowski, Tan, Xu, Gowda, Chen, Zhu, Tang, Qian, Zhu, and Hou]{wu2025llamarldistributedasynchronousreinforcement}
Bo~Wu, Sid Wang, Yunhao Tang, Jia Ding, Eryk Helenowski, Liang Tan, Tengyu Xu, Tushar Gowda, Zhengxing Chen, Chen Zhu, Xiaocheng Tang, Yundi Qian, Beibei Zhu, and Rui Hou.
\newblock Llamarl: A distributed asynchronous reinforcement learning framework for efficient large-scale llm training, 2025{\natexlab{a}}.
\newblock \url{https://arxiv.org/abs/2505.24034}.

\bibitem[Wu et~al.(2025{\natexlab{b}})Wu, Wang, Tang, Ding, Helenowski, Tan, Xu, Gowda, Chen, Zhu, et~al.]{wu2025llamarl}
Bo~Wu, Sid Wang, Yunhao Tang, Jia Ding, Eryk Helenowski, Liang Tan, Tengyu Xu, Tushar Gowda, Zhengxing Chen, Chen Zhu, et~al.
\newblock Llamarl: A distributed asynchronous reinforcement learning framework for efficient large-scale llm trainin.
\newblock \emph{arXiv preprint arXiv:2505.24034}, 2025{\natexlab{b}}.

\bibitem[Yang et~al.(2025)Yang, Li, Yang, Zhang, Hui, Zheng, Yu, Gao, Huang, Lv, Zheng, Liu, Zhou, Huang, Hu, Ge, Wei, Lin, Tang, Yang, Tu, Zhang, Yang, Yang, Zhou, Zhou, Lin, Dang, Bao, Yang, Yu, Deng, Li, Xue, Li, Zhang, Wang, Zhu, Men, Gao, Liu, Luo, Li, Tang, Yin, Ren, Wang, Zhang, Ren, Fan, Su, Zhang, Zhang, Wan, Liu, Wang, Cui, Zhang, Zhou, and Qiu]{yang2025qwen3technicalreport}
An~Yang, Anfeng Li, Baosong Yang, Beichen Zhang, Binyuan Hui, Bo~Zheng, Bowen Yu, Chang Gao, Chengen Huang, Chenxu Lv, Chujie Zheng, Dayiheng Liu, Fan Zhou, Fei Huang, Feng Hu, Hao Ge, Haoran Wei, Huan Lin, Jialong Tang, Jian Yang, Jianhong Tu, Jianwei Zhang, Jianxin Yang, Jiaxi Yang, Jing Zhou, Jingren Zhou, Junyang Lin, Kai Dang, Keqin Bao, Kexin Yang, Le~Yu, Lianghao Deng, Mei Li, Mingfeng Xue, Mingze Li, Pei Zhang, Peng Wang, Qin Zhu, Rui Men, Ruize Gao, Shixuan Liu, Shuang Luo, Tianhao Li, Tianyi Tang, Wenbiao Yin, Xingzhang Ren, Xinyu Wang, Xinyu Zhang, Xuancheng Ren, Yang Fan, Yang Su, Yichang Zhang, Yinger Zhang, Yu~Wan, Yuqiong Liu, Zekun Wang, Zeyu Cui, Zhenru Zhang, Zhipeng Zhou, and Zihan Qiu.
\newblock Qwen3 technical report, 2025.
\newblock \url{https://arxiv.org/abs/2505.09388}.

\bibitem[Zheng et~al.(2025{\natexlab{a}})Zheng, Liu, Li, Chen, Yu, Gao, Dang, Liu, Men, Yang, et~al.]{zheng2025group}
Chujie Zheng, Shixuan Liu, Mingze Li, Xiong-Hui Chen, Bowen Yu, Chang Gao, Kai Dang, Yuqiong Liu, Rui Men, An~Yang, et~al.
\newblock Group sequence policy optimization.
\newblock \emph{arXiv preprint arXiv:2507.18071}, 2025{\natexlab{a}}.

\bibitem[Zheng et~al.(2025{\natexlab{b}})Zheng, Zhao, and Chen]{zheng2025prosperity}
Haizhong Zheng, Jiawei Zhao, and Beidi Chen.
\newblock Prosperity before collapse: How far can off-policy rl reach with stale data on llms?
\newblock \emph{arXiv preprint arXiv:2510.01161}, 2025{\natexlab{b}}.

\bibitem[Zheng et~al.(2025{\natexlab{c}})Zheng, Zhou, Bartoldson, Kailkhura, Lai, Zhao, and Chen]{zheng2025act}
Haizhong Zheng, Yang Zhou, Brian~R Bartoldson, Bhavya Kailkhura, Fan Lai, Jiawei Zhao, and Beidi Chen.
\newblock Act only when it pays: Efficient reinforcement learning for llm reasoning via selective rollouts.
\newblock \emph{arXiv preprint arXiv:2506.02177}, 2025{\natexlab{c}}.

\end{thebibliography}
